\documentclass[11pt]{article}
\usepackage[a4paper,margin=27mm,headheight=14pt]{geometry}
\usepackage[T1]{fontenc}
\usepackage{lmodern}
\usepackage{amsmath,amssymb,amsthm,bm}
\usepackage{microtype,graphicx,booktabs,tabularx,longtable}
\usepackage[numbers,sort&compress]{natbib}
\usepackage[hidelinks]{hyperref}
\usepackage{fancyhdr}
\hypersetup{pdftitle={Approximation Property of Dropout Neural Networks: Sobolev Rates and Confidence Bounds},pdfauthor={Jia-He Yao}}
\newtheorem{theorem}{Theorem}[section]
\newtheorem{lemma}[theorem]{Lemma}
\newtheorem{corollary}[theorem]{Corollary}

\theoremstyle{definition}
\newtheorem{definition}[theorem]{Definition}
\theoremstyle{remark}

\newcommand{\F}{\mathcal F}

\title{Approximation Property of Dropout Neural Networks:\\
Sobolev Rates and Confidence Bounds}
\author{Jia-He Yao\\[0.3em]
\small\href{mailto:yjher0101@gmail.com}{\texttt{yjher0101@gmail.com}}}
\date{20 September 2026}
\begin{document}
\maketitle
\begin{abstract}
The universal approximation property of dropout neural networks does not
by itself describe the network size required for an accurate random
realization. In this work, we study approximation of the unit ball of
$W^{n,\infty}([0,1]^d)$ by ReLU networks whose edges are retained
independently with probability $p$. The approximation error is measured
uniformly over the input domain, and the guarantee holds with probability
at least $1-\delta$ for a single sampled network. We construct networks
of constant depth and size
$\widetilde O_{n,d}(p^{-9}\varepsilon^{-\max\{d/n,2\}}
\log(1/\delta))$. The construction combines bounded local subnetworks,
localization on a successful approximation event, and a multiscale Taylor
decomposition. Conversely, Sobolev capacity imposes a lower bound on the
number of surviving edges, while approximation of a fixed affine function
requires an output-layer cost of order
$((1-p)/p)\varepsilon^{-2}\log(1/\delta)$ at sufficiently high confidence.
For fixed $p\in(0,1)$ and $\delta<\min\{1/2,1-p\}$, the upper and lower
bounds match in the accuracy exponent under a fixed or logarithmic depth
budget. When $d\leq2n$, they also match in confidence up to logarithms of
accuracy. We extend the lower bounds to $W^{n,r}$ targets with $L^s$
error, and distinguish this extension from the upper bound for
$W^{n,\infty}$. The optimal retention dependence and logarithmic factors
remain open.
\end{abstract}
\noindent\textbf{Keywords:} dropout neural networks; random approximation;
Sobolev spaces; approximation rates; lower bounds.
\clearpage
\begingroup
\small
\tableofcontents
\endgroup
\clearpage
\section{Introduction}
\label{sec:introduction}

Neural networks have achieved remarkable success in various applications,
while the theoretical understanding of their approximation properties
continues to develop. The universal approximation theorem explains why
a sufficiently large network can approximate a prescribed function.
Approximation rates address a further question: how many weights and
computation units are needed to achieve a given accuracy? For ReLU
networks, this question has been studied in considerable detail for
Sobolev and related function spaces \citep{yarotsky2017,siegel2023optimal}.

Dropout introduces randomness into the computation of a neural network.
Even after the architecture and all weights have been fixed, two sampled
networks may produce different outputs. As a result, deterministic
approximation of the target function is only part of the problem. The
network must also be sufficiently redundant to control the variation
caused by random deletion. Understanding the interaction between these
two requirements is the main purpose of this work.

We consider independent edge dropout. Each edge is retained with
probability $p$, and the retained weights are used throughout the input
domain. Following the random-approximation formulation of
\citet{manita2022universal}, we seek a network $F_R$ such that
\[
 \mathbb P\left\{\sup_{x\in[0,1]^d}|F_R(x)-f(x)|
                \leq\varepsilon\right\}\geq1-\delta.
\]
Here the weights are chosen before the masks $R$ are sampled. Thus the
same random realization must approximate $f$ at every input. This
requirement is stronger than pointwise accuracy in expectation, and it
differs from replacing the masks by their expectations at prediction time.

For a Sobolev target with $n$ bounded weak derivatives in dimension $d$,
local Taylor approximation suggests a spatial cost of order
$\varepsilon^{-d/n}$. On the other hand, averaging independent random
contributions typically requires order $\varepsilon^{-2}$ copies.
It is therefore natural to ask whether dropout forces these costs to
multiply, or whether a construction can accommodate both at a cost
determined by the larger one. Our results show that the latter is possible,
up to logarithmic factors, when $p$ and a sufficiently high confidence
level are fixed.

The main upper bound has accuracy exponent $\max\{d/n,2\}$ and uses
depth bounded independently of $\varepsilon$. Its proof follows the local
Taylor approach, but the random implementation requires two changes.
First, local support is recovered only on the event where approximation
succeeds. Exact support preservation under every mask would be
incompatible with a nonzero ReLU output in the present model. Second,
we approximate differences between successive Taylor resolutions. The
fine-scale corrections are small, so they can be implemented with lower
relative precision. Summing within each scale before combining the scales
preserves this saving in the final network.

The lower bounds explain both parts of the exponent. A family of
localized Sobolev functions forces sufficiently many edges to survive
in each successful realization. In addition, even the affine function
$f_0(x)=(1+x_1)/2$ requires order
$\varepsilon^{-2}\log(1/\delta)$ incoming output edges at fixed $p$.
The latter statement allows arbitrary hidden depth and arbitrary real
weights. Its proof must use the probability of approximation directly:
a large variance alone would not imply a lower bound on the failure
probability.

Together these arguments give a match of the accuracy exponent for the
$W^{n,\infty}$ unit ball under a fixed or logarithmic depth budget.
For $d\leq2n$, the stochastic term also gives the correct dependence on
$\log(1/\delta)$ up to accuracy logarithms. For $d>2n$, the corresponding
joint lower bound is not established. We state these distinctions
explicitly, as well as the difference between our lower-bound extension
to finite Sobolev integrability and an upper bound for that larger class.

We first discuss related work and the mechanisms behind the rate.
Section~\ref{sec:model-main} gives the precise model and main results.
The lower and upper bounds are proved in Sections~\ref{sec:lower-main}
and~\ref{sec:multiscale}, respectively, using the elementary random
modules of Section~\ref{sec:modules-main}. The later sections discuss
extensions and remaining questions. A separate account in
Section~\ref{sec:development} records the changes from the supplied
unfinished drafts and the role of AI assistance in the present manuscript.

\section{Related Work}
\label{sec:related}

Dropout was developed as a regularization method based on randomly
removing activations \citep{srivastava2014dropout}. Randomly removing
weights is commonly called DropConnect \citep{wan2013dropconnect}.
We use the term independent edge dropout to specify this latter type
of masking. In particular, masks on edges incident to the same neuron
are independent, rather than being shared through one node mask.

The approximation of smooth functions by deterministic ReLU networks
provides the starting point of our analysis. \citet{yarotsky2017} uses
local Taylor polynomials and ReLU implementations of arithmetic
operations to obtain quantitative approximation bounds. We retain this
local representation, but independent deletion of every connecting edge
prevents a direct use of the deterministic implementation. In particular,
a connection that merely copies an intermediate value must also be
implemented with its own random edges.

For lower bounds, \citet{bartlett2019nearly} bounds the VC dimension of
piecewise linear networks in terms of the number of parameters, depth,
and computation units. These estimates apply to each deterministic
realization of a dropout network. We use them after removing constant
parts of the realized graph, so that the relevant complexity is the
number of retained edges. The extension to finite $L^s$ error uses the
translation-invariant approximation theorem of
\citet[Theorem 22]{siegel2023optimal}. That result is a deterministic
Sobolev approximation tool and is not a dropout-specific lower bound.

\citet{manita2022universal} establishes universal approximation for
networks with random filters, including a mode in which the network
remains random when it is evaluated. Their Proposition~9 gives an
algebraic identity expressing a deterministic network through expectations
of masked subnetworks. We use this identity for a fixed-size local window
and a clipped coordinate. Because all of their masked subnetworks are
bounded and have finitely many common breakpoints, concentration can be
made uniform without paying for the slope of the window. Applying the
same subset expansion to the entire approximation network would instead
have cost exponential in its number of edges. The quantitative upper
bound therefore depends on where this existing identity is used.

Fault tolerance gives another perspective on random deletion.
\citet{elmhamdi2017neurons} studies how failures propagate through a
given network, including extensions to failed synapses.
\citet{elmhamdi2019probabilistic} studies small random failures under
a continuous-limit description. These questions are related to ours,
but we optimize over networks for an entire Sobolev ball and require a
single realization to succeed uniformly over its inputs. The
information-contraction bounds for noisy binary networks in
\citet{zhou2021information} concern a different activation and noise
model, so they do not directly give the output-layer bound used here.

It is also useful to distinguish approximation from learning and
optimization. \citet{senencerda2025almost} studies convergence of dropout
training algorithms. The dropout analysis in
\citet{wu2026deep} concerns statistical estimation in a spatially varying
coefficient model. Neither type of guarantee by itself specifies the
minimum size for the random uniform approximation problem considered
here. Covering estimates for bounded, sparse, or quantized deterministic
networks, such as those of \citet{ou2026covering}, likewise require care
when compared with our class, whose weights are unrestricted real numbers.

The present results combine these established tools with explicit
constructions and lower bounds for independent edge dropout. We do not
claim that the algebraic debiasing identity, deterministic capacity
estimates, or local Taylor representation are new. The results to be
assessed are the stated quantitative bounds under the common model of
Section~\ref{sec:model-main}. Their publication priority is not settled
by the literature comparisons above.

\section{Theoretical Contributions and Common Mechanisms}
\label{sec:mechanisms}

The rate can be understood through the relation between spatial
resolution and random fluctuation. We describe this relation before
giving the formal statements, since it also explains why an accurate
implementation of each local term separately can be unnecessarily costly.

\subsection{Spatial complexity and random fluctuation}

A resolution $h$ requires order $h^{-d}$ local pieces and gives a Taylor
remainder of order $h^n$. Thus resolving the target to error
$\varepsilon$ suggests $h\asymp\varepsilon^{1/n}$ and a spatial count
$\varepsilon^{-d/n}$. This is reflected in the deterministic capacity
part of our lower bound. Dropout adds another restriction. Conditional on
the hidden masks, an affine output is a sum of independently retained
terms. Requiring this sum to reproduce a nonzero change of the target
with small error imposes a redundancy cost, even if the hidden network
already contains a perfect representation of the target.

These are two necessary costs for the worst-case approximation problem.
Their maximum is a lower bound; their product does not follow. The upper
construction shows how to meet both requirements at the larger polynomial
order. When $d<2n$, greater smoothness decreases the spatial difficulty
but does not remove the output fluctuation. When $d>2n$, the number of
spatial features becomes dominant. The boundary $d=2n$ is where the two
contributions have the same polynomial size.

\subsection{Localization on the event that matters}

A deterministic Taylor partition has bounded overlap. To use that overlap
in a random construction, a branch must not create a small error at every
point outside its intended region: many such errors could accumulate.
Requiring exact zero output outside the region for every mask is too
strong. In Section~\ref{sec:modules-main}, we prove that a nonzero
ReLU-output network cannot satisfy the required all-mask localization.

There is, however, a simple way to recover the needed property on the
successful event. If a nonnegative target $q$ is approximated globally by
$G=\sigma(Z)$ with error at most $\eta$, then changing its last bias gives
\[
 H=\sigma(Z-\eta),\qquad
 0\leq H\leq q,\qquad |H-q|\leq2\eta.
\]
Consequently $q(x)=0$ implies $H(x)=0$ on that event. Only the bounded
number of overlapping branches can then contribute error. The outer
module is controlled uniformly on a set containing all successful inner
outputs, which also makes composition valid at random inputs. This
argument applies to nonnegative localized decompositions beyond Taylor
polynomials whenever a suitable global local-module approximation is
available.

\subsection{Why a narrow window need not be expensive}

A window supported on a region of width $h$ has slopes of order $h^{-1}$.
Approximating these large affine coefficients by independent averaging
would introduce an $h^{-2}$ cost. Instead, consider the entire small
window circuit before making it reliable. Although its first-layer
weights can be large, every masked subnetwork has output between zero
and one. Its number of affine pieces is bounded independently of $h$.

The debiasing identity of \citet{manita2022universal}, applied only to
this small circuit, provides an unbiased random representation. A union
bound over the common breakpoints turns scalar concentration into a
uniform estimate. The number of breakpoints, rather than their locations
or the slope, determines the cost. This observation separates spatial
resolution from the number of copies required for reliability. It uses
the absence of a weight-magnitude constraint; it does not provide a
bounded-weight or finite-bit implementation.

\subsection{Small corrections require fewer copies}

Let $F_j$ be the local Taylor approximation at resolution $2^{-j}$.
Instead of implementing only $F_J$, write
\[
 F_J=F_0+\sum_{j=1}^J(F_j-F_{j-1}).
\]
After localization and normalization, the $j$th difference has
$O(2^{jd})$ terms of amplitude $O(2^{-jn})$. If a normalized random
module costs order $\eta^{-2}$ to achieve error $\eta$, then an actual
term of amplitude $a_j$ needs only relative error approximately
$\varepsilon/((J+1)a_j)$. Its cost is proportional to
$a_j^2\varepsilon^{-2}(J+1)^2$. The resulting sum is
\[
 \varepsilon^{-2}(J+1)^2\sum_{j=0}^J2^{j(d-2n)},
\]
apart from retention and logarithmic factors. This is the source of
$\max\{d/n,2\}$.

There is a useful general lesson in this calculation. For a localized
decomposition with $O(2^{jd})$ terms of amplitude $O(2^{-jn})$, a local
accuracy cost $\eta^{-\beta}$ would give the formal sum
$\varepsilon^{-\beta}\sum_j2^{j(d-\beta n)}$, with logarithmic losses
from sharing the error budget. The corresponding polynomial exponent
would be $\max\{d/n,\beta\}$. This is an accounting principle, conditional
on uniform local modules, bounded overlap, and a summation construction
that preserves the saving. We verify these conditions for $\beta=2$ in
the present model; we do not assert a theorem for other noise models.

The last condition is substantive. We first sum the positive and negative
terms within each scale. The final module receives only $2(J+1)$ values.
If every local branch were instead connected to one final summation module
at tolerance $\varepsilon$, a factor equal to the total number of branches
would return and the multiscale saving would be lost.

\begin{figure}[t]
 \centering
 \includegraphics[width=0.94\linewidth]{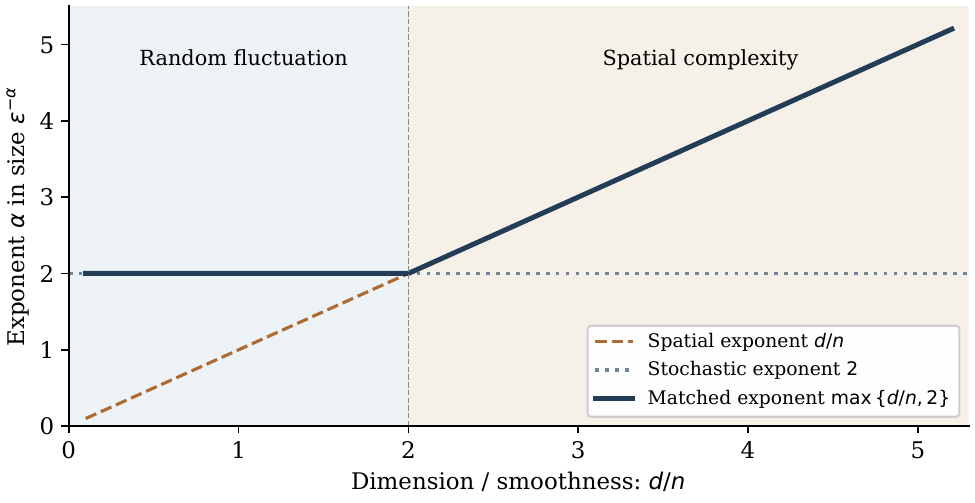}
 \caption{Accuracy exponents at fixed retention and sufficiently high
 fixed confidence. The spatial and stochastic obstructions lead to the
 exponent $\max\{d/n,2\}$ under the stated depth budget. The diagram
 displays polynomial exponents only; it does not suppress the conditions
 of the matching theorem or imply equality of logarithmic factors.}
 \label{fig:exponents}
\end{figure}

\section{Preliminaries and Main Results}
\label{sec:model-main}

Let $\sigma(t)=\max\{t,0\}$ be the ReLU activation, and let
$\operatorname{Bernoulli}(p)$ denote the distribution on $\{0,1\}$
with mean $p$. Throughout this work, $p$ is the retention probability.

\begin{definition}[Independent edge dropout]
A dropout network is a finite feed-forward graph whose hidden units have
the form $z_v=\sigma(b_v+\sum_{u\to v}w_{uv}R_{uv}z_u)$.
Every present edge has its own independent
$R_{uv}\sim\operatorname{Bernoulli}(p)$, including the edges into the
scalar output. Biases are deterministic and unmasked. The output is
affine unless a ReLU output is explicitly specified.
\end{definition}

The architecture and weights are fixed before sampling $R$, and one
sample defines a function $F_R$ on all inputs. There is no protected
input-dependent path to the output. The graph may have skip connections;
the upper constructions use only adjacent layers. For $E$ edges and $U$
non-input units, including the output, define the size as $S=E+U+1$.
Thus biases and fixed construction weights are counted. Depth is the
largest number of successive computational layers, including the output.
Weights have no magnitude or bit-precision restriction.

For a multi-index $\alpha\in\mathbb N_0^d$, write $|\alpha|=\sum_i\alpha_i$
and $\alpha!=\prod_i\alpha_i!$. The target class is
\[
 \F_{n,d}=\left\{f\in W^{n,\infty}((0,1)^d):
       \max_{|\alpha|\leq n}\|D^\alpha f\|_{L^\infty}\leq1\right\}.
\]
Derivatives are weak derivatives. We use the $C^{n-1,1}$ representative
on $[0,1]^d$, so values and derivatives below order $n$ at boundary
centers have their continuous meanings. This choice also identifies the
essential-supremum error with the supremum error for our approximants.
An affine output is necessary for the signed class, since a ReLU output
cannot approximate the constant function $-1$ to error less than one.

\begin{definition}[Worst-case size budget]
For $L\geq1$, let $S_*(\varepsilon,\delta;p,L)$ be the smallest size
budget for which
\begin{equation}\label{eq:lb-success}
 \forall f\in\F_{n,d}\quad\exists(A,w):\quad
 \operatorname{size}(A)\leq S_*,\quad \operatorname{depth}(A)\leq L,
 \quad \mathbb P\{\|F_{A,w,R}-f\|_\infty\leq\varepsilon\}\geq1-\delta.
\end{equation}
The architecture is allowed to depend on $f$. If no finite budget exists,
the value is $+\infty$.
\end{definition}

We write $a=d/n$, $\ell_\varepsilon=\log(2/\varepsilon)$, and
$\gamma_p=\min\{1/2,1-p\}$. Constants with subscripts $n,d$ may depend
on these fixed integers. The notation $\widetilde O$ or
$\widetilde\Theta$ hides powers of $\ell_\varepsilon$ only, unless the
parameters being fixed are stated explicitly.

\begin{theorem}[Upper and lower size bounds]\label{thm:main}
For each $n,d\geq1$ there is a constant depth $L_{n,d}$ with the following
properties. For $0<p\leq1$, $0<\varepsilon,\delta\leq1/2$, and
$L\geq L_{n,d}$,
\begin{equation}\label{eq:main-upper}
 S_*(\varepsilon,\delta;p,L)
 \leq C_{n,d}p^{-9}\varepsilon^{-2}(J+1)^2
 \left(\sum_{j=0}^J2^{j(d-2n)}\right)
 \log\frac{C_{n,d}}{\varepsilon\delta},
\end{equation}
where $J=\lceil n^{-1}\log_2(2C^{\rm Tay}_{n,d}/\varepsilon)\rceil$
and $C^{\rm Tay}_{n,d}\geq1$ depends only on $n,d$.

For $0<p<1$, $0<\delta<\gamma_p$, and sufficiently small
$\varepsilon$ depending on $n,d$ and satisfying
$\varepsilon\leq\min\{1/1024,(1-p)/(512p)\}$, every $L\geq1$ obeys
\begin{equation}\label{eq:main-lower}
 S_*(\varepsilon,\delta;p,L)\geq c_{n,d}
 \max\left\{\frac1p\max\left\{
 \frac{\varepsilon^{-d/n}}{L\ell_\varepsilon},
 \varepsilon^{-d/(2n)}\right\},\quad
 \frac{1-p}{p}\varepsilon^{-2}\log\frac1\delta\right\}.
\end{equation}
\end{theorem}

The upper bound is proved in Theorem~\ref{thm:multiscale-upper}.
The lower bound follows from Theorems~\ref{thm:sobolev-survival}
and~\ref{thm:confidence-extended}; the combination is justified at the
end of Section~\ref{sec:lower-main}.

\begin{corollary}[Matching accuracy exponent]\label{thm:matching-exponent}
Fix $n,d$, $p\in(0,1)$, and $0<\delta<\gamma_p$. If
$L_{n,d}\leq L(\varepsilon)\leq C_0\ell_\varepsilon$ for all
sufficiently small $\varepsilon$, then
\begin{equation}\label{eq:matching-exponent}
 S_*(\varepsilon,\delta;p,L(\varepsilon))
 =\widetilde\Theta_{n,d,p,\delta,C_0}
       (\varepsilon^{-\max\{d/n,2\}}).
\end{equation}
In particular, this holds for any fixed depth cap at least $L_{n,d}$.
\end{corollary}
\begin{proof}
The lower bound is at least a constant times
$\max\{\varepsilon^{-d/n}/\ell_\varepsilon^2,\varepsilon^{-2}\}$.
In the upper bound, the geometric sum is bounded for $d<2n$, equals
$J+1$ for $d=2n$, and is $O_{n,d}(\varepsilon^{-(d-2n)/n})$ for $d>2n$.
All remaining factors at fixed $p,\delta$ are powers of
$\ell_\varepsilon$. For a fixed depth cap the capacity denominator
improves to $\ell_\varepsilon$.
\end{proof}

\begin{corollary}[Joint accuracy and confidence for $d\leq2n$]
\label{cor:joint-confidence}
Fix $n,d,p,C_0$ with $d\leq2n$ and $p\in(0,1)$, and use the same depth
range as in Corollary~\ref{thm:matching-exponent}. Uniformly for
$0<\delta<\gamma_p$ and sufficiently small $\varepsilon$,
\begin{equation}\label{eq:sobolev-joint-match}
 S_*(\varepsilon,\delta;p,L(\varepsilon))
 =\widetilde\Theta_{n,d,p,C_0}
       \left(\varepsilon^{-2}\log\frac1\delta\right).
\end{equation}
Only accuracy logarithms are hidden here.
\end{corollary}
\begin{proof}
The output term gives the lower bound. The upper bound is at most
\[
 C_{n,d,p}\varepsilon^{-2}\ell_\varepsilon^4\log(1/\delta).
\]
Indeed, for $\delta<1/2$ we have
\[
 \log\frac{C_{n,d}}{\varepsilon\delta}
 \leq C_{n,d}\ell_\varepsilon\log\frac1\delta.
\]
\end{proof}

For fixed $p,\delta$ the explicit upper bounds are
$O(\varepsilon^{-2}\ell_\varepsilon^3)$ when $d<2n$,
$O(\varepsilon^{-2}\ell_\varepsilon^4)$ when $d=2n$, and
$O(\varepsilon^{-d/n}\ell_\varepsilon^3)$ when $d>2n$.
Thus matching the exponent does not identify the optimal logarithms.
Neither corollary asserts optimal dependence on $p$, an unrestricted-depth
rate, or a matching upper bound for general finite $r$ in $W^{n,r}$.

\section{Lower Bounds for Random Approximation}
\label{sec:lower-main}
We first isolate the cost of random output edges. The argument uses only
one affine target and is independent of the hidden depth. We then turn
to the entire Sobolev ball and count the edges that a successful mask
must retain.
\subsection{The output-layer obstruction}
The following two lemmas convert a sufficiently concentrated Bernoulli
sum into bounds on its coefficients and variance. This step is needed
before a change-of-measure argument can be applied.
\begin{lemma}[Removing large coefficients without lowering success]
\label{lem:trim-bernoulli}
Let $T=\sum_{i=1}^m a_iB_i$, where $B_i$ are independent Bernoulli$(p)$,
and let $I=[t-\eta,t+\eta]$, $\eta>0$. If
$\mathbb P(T\in I)=s>\max\{1/2,p\}$, there is a subset $J$ such that
\[
 |a_i|\leq2\eta\quad(i\in J),\qquad
 \mathbb P\left(\sum_{i\in J}a_iB_i\in I\right)\geq s.
\]
The center $t$ is unchanged. The coefficients may have either sign.
\end{lemma}
\begin{proof}
If $|a_i|>2\eta$, write $T=T_0+a_iB_i$ and set
$u=\mathbb P(T_0\in I)$, $v=\mathbb P(T_0\in I-a_i)$.
The two intervals are disjoint, so $u+v\leq1$ and
$s=(1-p)u+pv$.
If $p\geq1/2$, this implies $s\leq p$, a contradiction. Thus in that
case no large coefficient exists. If $p<1/2$, then
\[
 s\leq p+(1-2p)u,
 \qquad u\geq\frac{s-p}{1-2p}>s,
\]
where the last inequality uses $s>1/2$. Delete this coefficient, equivalently
condition its mask to be zero, and repeat. The remaining masks keep their
independent Bernoulli$(p)$ law. The finite procedure proves the claim.
\end{proof}

\begin{lemma}[A small-interval variance bound]
\label{lem:sinc-concentration}
Let $X=\sum_i a_iB_i$ with independent Bernoulli$(p)$ masks and
$|a_i|\leq2\eta$. If $\mathbb P(|X-t|\leq\eta)>1/2$, then
\begin{equation}\label{eq:sinc-moments}
 V:=\operatorname{Var}(X)\leq64\eta^2,
 \qquad |\mathbb EX-t|\leq9\eta.
\end{equation}
\end{lemma}
\begin{proof}
If $V=0$, the claims are immediate. Otherwise let
$K(u)=(\sin u/u)^2$, with $K(0)=1$. For $|u|\leq1/2$,
$K(u)\geq(23/24)^2>8/9$. The elementary Fourier identity
\[
 K(u)=\frac12\int_{-2}^2(1-|z|/2)e^{izu}\,dz
\]
therefore gives, writing $\phi_X$ for the characteristic function,
\[
 \mathbb P(|X-t|\leq\eta)
 \leq\frac98\mathbb E K\left(\frac{X-t}{2\eta}\right)
 \leq\frac98\eta\int_{-1/\eta}^{1/\eta}|\phi_X(z)|\,dz.
\]
For $|z|\leq1/\eta$, one has $|a_i z/2|\leq1$ and
$|\sin(a_i z/2)|\geq(5/6)|a_i z/2|$. With $q=1-p$,
\[
 |1-p+pe^{ia_i z}|^2=1-4pq\sin^2(a_i z/2),
 \qquad
 |\phi_X(z)|\leq\exp(-25Vz^2/72).
\]
Extending the integral to the real line yields
\[
 \mathbb P(|X-t|\leq\eta)
 \leq\frac98\sqrt{\frac{72\pi}{25}}\frac{\eta}{\sqrt V}.
\]
Since the left side exceeds $1/2$, it follows that
$V\leq(729\pi/50)\eta^2<64\eta^2$.
For the mean, if $|\mathbb EX-t|>\eta$, apply Cantelli's inequality on
the side containing the interval. Its probability greater than $1/2$
forces $(|\mathbb EX-t|-\eta)^2<V$. This proves the second bound.
\end{proof}

\begin{theorem}[Output confidence bound for success above $\max\{1/2,p\}$]
\label{thm:confidence-extended}
Suppose that the retention, confidence, and accuracy satisfy
\begin{equation}\label{eq:extended-confidence-range}
 0<\delta<\min\{1/2,1-p\},\qquad
 0<\varepsilon\leq\min\{1/1024,(1-p)/(512p)\}.
\end{equation}
Specifically, any network approximating $f_0(x)=(1+x_1)/2$ uniformly
with probability at least $1-\delta$ has at least
\begin{equation}\label{eq:extended-output-bound}
 m\geq\frac{1-p}{2^{19}p\varepsilon^2}\log\frac1\delta
\end{equation}
incoming output edges. Hidden depth and weights are unrestricted, and
the output may be affine or ReLU.
\end{theorem}
\begin{proof}
By averaging the conditional success probability over the finitely many
hidden masks, choose one with success probability at least $1-\delta$.
The output preactivation then has the form
$Z(x)=b+\sum_{i=1}^m w_iB_i h_i(x)$ with deterministic $h_i$
and independent output masks $B_i$. For a ReLU output, success and
$f_0\geq1/2>\varepsilon$ imply $F_R=Z$ at both $0$ and $e_1$.
For an affine output this equality is automatic. Set
$a_i=w_i(h_i(e_1)-h_i(0))$. Subtracting the two endpoint values
cancels the unmasked bias and gives $T=\sum a_iB_i$ with
$\mathbb P(|T-1/2|\leq2\varepsilon)\geq1-\delta>\max\{1/2,p\}$.
Lemma~\ref{lem:trim-bernoulli} supplies a sub-sum $X$ with the same
success guarantee and coefficients at most $4\varepsilon$ in absolute
value. Lemma~\ref{lem:sinc-concentration}, with $\eta=2\varepsilon$,
gives
\[
 V\leq256\varepsilon^2,
 \qquad |\mu-1/2|\leq18\varepsilon,
 \qquad \mu\geq1/4.
\]
Under the auxiliary retention $r=p(1-512\varepsilon)$,
$p/2\leq r<p$, $r(1-r)/(p(1-p))\leq2$, and hence
\[
 \mathbb E_rX\leq\tfrac12-110\varepsilon,
 \qquad \operatorname{Var}_r(X)\leq512\varepsilon^2.
\]
Consequently the interval $[1/2-2\varepsilon,1/2+2\varepsilon]$ has
probability at most
$512/(512+108^2)<1/16$ under the auxiliary law. Write $P_p,P_r$ for
the product laws of the remaining masks at the two retentions.
If $A$ is the complement of this interval event,
then $P_p(A)\leq\delta$ and $P_r(A)\geq15/16$.
For $b\geq15/16$, the binary entropy
$h(b)=-b\log b-(1-b)\log(1-b)$ satisfies
$h(b)\leq h(15/16)<(7/16)\log2$. Thus, since $\delta<1/2$,
\[
 D(P_r\Vert P_p)\geq b\log(1/\delta)-h(b)
                  \geq\tfrac12\log(1/\delta).
\]
On the other hand, additivity of relative entropy and the elementary
Bernoulli bound $D(\operatorname{Bernoulli}(r)\Vert
\operatorname{Bernoulli}(p))\leq(r-p)^2/[p(1-p)]$ give
\[
 D(P_r\Vert P_p)\leq |J|\frac{(r-p)^2}{p(1-p)}
 \leq m\,\frac{2^{18}p\varepsilon^2}{1-p}.
\]
Combining the two estimates proves the bound. Removing
coefficients was only a proof device; no assumption that arbitrary
output-edge deletions preserve success has been made.
\end{proof}
The strict threshold cannot in general be removed for $p\geq1/2$:
the single-edge network $F_B(x)=1/2+(B/2)x_1$ is exactly $f_0$ when
$B=1$, with probability $p$, at constant size for every accuracy.
This also proves sharpness of the confidence threshold at $p=1/2$ for
the affine-target obstruction. It is not a counterexample to the
Sobolev capacity lower bound, which uses other hard functions.
For $p<1/2$, optimality of the new confidence threshold is not asserted.

For this particular target the order in accuracy and confidence is
attainable at fixed retention. With independent Bernoulli$(p)$ variables
$B_j,C_j$, set
\[
 \widehat f_R(x)=\frac12+\frac1{2mp^2}
                       \sum_{j=1}^m C_j\sigma(B_jx_1).
\]
The uniform error is
$\frac12|\operatorname{Bin}(m,p^2)/(mp^2)-1|$. For
$0<\varepsilon\leq1/2$, the multiplicative Chernoff bound gives
\[
 \mathbb P\{\|\widehat f_R-f_0\|_\infty>\varepsilon\}
 \leq 2\exp(-4mp^2\varepsilon^2/3).
\]
Thus $m=\lceil3\log(2/\delta)/(4p^2\varepsilon^2)\rceil$ suffices.
This comparison is sharp in $\varepsilon,\delta$ at fixed $p$, not in
retention dependence.
\subsection{The Sobolev capacity obstruction}

\begin{theorem}[Capacity lower bound]\label{thm:lb-capacity}
For fixed $n,d\geq1$, there are constants $c,\varepsilon_0>0$ such that
any budgets satisfying \eqref{eq:lb-success}, with $\delta<1$ and
$0<\varepsilon<\varepsilon_0$, obey
\begin{equation}\label{eq:lb-capacity}
 S L\log(S+2)\geq c\varepsilon^{-d/n}.
\end{equation}
In particular, if $L\leq C_0\log(1/\varepsilon)$, then
\begin{equation}\label{eq:lb-logdepth}
 S\geq c'\frac{\varepsilon^{-d/n}}{[\log(1/\varepsilon)]^2}.
\end{equation}
This bound alone does not measure a cost specific to dropout.
\end{theorem}

\begin{proof}
Because the success probability is positive, each target has at least one
successful mask. Freezing that mask produces a deterministic ReLU network
of size at most $S$ and depth at most $L$. The successful mask can depend
on the target; no common successful mask is needed.

Choose a smooth bump $0\leq\varphi\leq1$ supported in $(-1/2,1/2)^d$
with $\varphi(0)=1$ and
$A\geq\max_{|\alpha|\leq n}\|D^\alpha\varphi\|_\infty$, $A\geq1$.
Let $h=(2A\varepsilon)^{1/n}$. For sufficiently small $\varepsilon$,
one can place $M\geq c_{n,d}\varepsilon^{-d/n}$ disjoint cubes of side $h$
inside $[0,1]^d$, with centers $z_1,\ldots,z_M$.
For each $\omega\in\{-1,1\}^M$ set
\[
 f_\omega(x)=\frac12+
       2\varepsilon\sum_{i=1}^M\omega_i\varphi((x-z_i)/h).
\]
The supports are disjoint. For $1\leq|\alpha|\leq n$,
$\|D^\alpha f_\omega\|_\infty\leq2\varepsilon A h^{-n}=1$;
also $\|f_\omega\|_\infty\leq1$ for $\varepsilon\leq1/4$.
Thus $f_\omega\in\mathcal F_{n,d}$ and
$f_\omega(z_i)=1/2+2\varepsilon\omega_i$.
A uniform $\varepsilon$ approximation recovers every sign by thresholding
the network at $1/2$. The deterministic network class therefore shatters
the $M$ centers.

For a fixed architecture, the ReLU VC-dimension bound of
\cite{bartlett2019nearly} is $O(SL\log(S+2))$.
Allowing the architecture to depend on the target does not change this order:
there are at most $\exp(C_dS\log(S+2))$ directed architectures with
at most $S$ units and edges, after labeling units in topological order.
For each architecture, Sauer's bound on $M$ points is at most
$(eM/V)^V$ when $M\geq V$, where $V=C_dSL\log(S+2)$ is a common
upper bound. Taking the union multiplies this by the number of architectures.
If the union shatters $M$ points, then
\[
 M\log2\leq C_dS\log(S+2)+V\log(eM/V),
\]
which implies $M=O(V)$. The case $M<V$ is immediate. This proves
\eqref{eq:lb-capacity}.

For the depth specialization, if $S\geq\varepsilon^{-d/n}$ there is
nothing to prove. Otherwise $\log(S+2)=O_{n,d}(\log(1/\varepsilon))$;
substitute this and the depth budget into \eqref{eq:lb-capacity}.
\end{proof}

\subsection{Counting retained edges}
\begin{lemma}[Deterministic complexity after masking]
\label{lem:active-edge-collapse}
A realization with $k$ retained edges can be represented deterministically
using at most $k$ edges and $k+1$ non-input units, without increasing depth.
Consequently the class of all such realizations of depth at most $L$,
thresholded at $1/2$, has VC dimension at most
\begin{equation}\label{eq:active-vc}
 C_d\min\{(k+1)L\log(k+2),\ (k+1)^2\}.
\end{equation}
The original number of units, including units with unmasked biases, is irrelevant
to this bound.
\end{lemma}
\begin{proof}
In the retained graph, a unit with no directed path from an input has a
constant activation. Remove all these units and absorb their contributions
to the other preactivations into the corresponding deterministic biases.
If the output has no input path, one constant output suffices. Otherwise,
each remaining non-input unit has an incoming retained edge on an input
path. Charging each unit to one such edge gives at most $k$ units, since
edges with different heads are distinct. Remove units that do not feed
the output as well. These operations create no edges and increase no depth.

For each resulting architecture, the parameter count is $O(k+1)$, and
the number $U_{\rm hid}$ of nonlinear hidden units is $O(k+1)$.
The two VC bounds from \cite[Theorems 7 and 10]{bartlett2019nearly},
written as $O(WL\log(W+2))$ and $O(W(U_{\rm hid}+1))$ to include
the affine output and the affine-only case, therefore give
the two terms in \eqref{eq:active-vc}. There are at most
$\exp(C_d(k+1)\log(k+2))$ possible architectures, including all depths
and skip patterns. Applying the same Sauer-and-union argument as in
Theorem~\ref{thm:lb-capacity} absorbs this architecture count into either
term. A constant output is covered by the added $1$.
\end{proof}

\begin{theorem}[Sobolev obstruction in terms of retained edges]
\label{thm:sobolev-survival}
Put $a=d/n$ and $\ell_\varepsilon=\log(2/\varepsilon)$.
For sufficiently small $\varepsilon>0$ and every integer $L\geq1$, there
are an integer $K=K_{n,d}(\varepsilon,L)\geq1$ and a target
$f_{\varepsilon,L}\in\mathcal F_{n,d}$ such that
\begin{equation}\label{eq:sobolev-critical-edges}
 K\asymp_{n,d}
 \max\left\{\frac{\varepsilon^{-a}}{L\ell_\varepsilon},
                    \varepsilon^{-a/2}\right\},
\end{equation}
and every successful mask for this target retains at least $K$ edges.
In particular, any network of depth at most $L$ and with $E$ original
edges achieving success probability at least $1-\delta$ must satisfy
\begin{equation}\label{eq:sobolev-binomial}
 \mathbb P\{\operatorname{Bin}(E,p)<K\}\leq\delta,
 \qquad
 E\geq\frac{(1-\delta)K}{p}.
\end{equation}
The target and $K$ can be chosen independently of $p$ and $\delta$.
Thus \eqref{eq:sobolev-binomial} is a necessary condition for a uniform
size budget over the entire Sobolev ball, for every $0<\delta<1$.
\end{theorem}
\begin{proof}
Use the $M\geq c_{n,d}\varepsilon^{-a}$ disjoint smooth bumps and the
$2^M$ functions $f_\omega$ constructed in Theorem~\ref{thm:lb-capacity}.
An $\varepsilon$ approximation recovers all $M$ signs at the bump centers.

Set $A=\max\{\varepsilon^{-a}/(L\ell_\varepsilon),
\varepsilon^{-a/2}\}$ and choose $K=\lfloor c_* A\rfloor$, where
$c_*>0$ is a sufficiently small constant depending only on $n,d$.
Decrease the allowed upper bound on $\varepsilon$ so that $c_* A\geq2$.
For $k=K-1$, the right side of \eqref{eq:active-vc} is strictly smaller
than $M$. Indeed, if the first term defining $A$ is larger, use
$\log(K+1)=O_{n,d}(\ell_\varepsilon)$ and the first VC bound. If the
second is larger, use the second VC bound. The class of all deterministic
networks with at most $K-1$ edges and depth at most $L$ therefore fails to
realize at least one of the $2^M$ sign patterns. Choose its target
$f_{\varepsilon,L}$. It cannot be approximated by that class.

For any proposed dropout approximant, a mask retaining fewer than $K$
edges would give such a forbidden deterministic approximant, by
Lemma~\ref{lem:active-edge-collapse}. Thus success implies $N_R\geq K$,
where $N_R\sim\operatorname{Bin}(E,p)$. This proves the probability
inequality. Also
$pE=\mathbb E N_R\geq K\mathbb P(N_R\geq K)\geq K(1-\delta)$.
\end{proof}

\paragraph{Depth specializations.}
For $\delta\leq1/2$, the following lower bounds on the original size $S$
follow, with sufficiently small $\varepsilon$:
\begin{align}
 L\leq C_0\ell_\varepsilon:
 &\quad S\geq \frac{c_{n,d,C_0}}{p}
                    \frac{\varepsilon^{-d/n}}{\ell_\varepsilon^2},
                    \label{eq:sobolev-p-logdepth}\\
 L\leq L_0:
 &\quad S\geq \frac{c_{n,d,L_0}}{p}
                    \frac{\varepsilon^{-d/n}}{\ell_\varepsilon},
                    \label{eq:sobolev-p-fixeddepth}\\
 \text{arbitrary finite depth}:\qquad
 &\quad S\geq\frac{c_{n,d}}p\varepsilon^{-d/(2n)}.
                    \label{eq:sobolev-p-anydepth}
\end{align}
The last statement can also be proved directly with the depth-free
$O((k+1)^2)$ VC bound and a single hard target against networks of all
depths. The binomial constraint is stronger than its expectation
consequence and retains information about very small failure probabilities.
It does not assert that having $K$ retained edges is sufficient for success.

\begin{proof}[Proof of the lower bound in Theorem~\ref{thm:main}]
Theorem~\ref{thm:sobolev-survival} and $\delta<1/2$ imply
$S_*\geq E\geq K/(2p)$ for the Sobolev hard target.
Theorem~\ref{thm:confidence-extended} implies the output-layer term
for $f_0\in\F_{n,d}$. Every budget that works for the whole class must
work for both targets, giving their maximum in \eqref{eq:main-lower}.
The targets need not be the same, and this argument does not multiply
the two lower bounds.
\end{proof}

\section{Approximation of Basic Functions}
\label{sec:modules-main}
We construct the elementary operations used in the Sobolev approximation.
All input and output connections of a module belong to that module and
are masked. When modules are composed, there is no additional protected
edge between them. We write $\eta$ and $\zeta$ for a local error tolerance
and failure probability, respectively.
\subsection{Affine combinations}
\begin{lemma}[Coordinate concentration without an input cover]\label{lem:lb-affine-upper}
Let $T\geq1$, $x_i\geq0$, $\sum_{i=1}^{T}|a_i|x_i\leq\tau$, $\tau>0$, and
$0<\varepsilon\leq\tau$. Both $\sum_i a_ix_i+b$ and
$\operatorname{ReLU}(\sum_i a_ix_i+b)$ admit uniform random approximation
with failure probability at most $\delta$, using one hidden layer and
\[
 O\left(T+\frac{T\tau^2}{p^2\varepsilon^2}
                 \log\frac{2T}{\delta}\right)
\]
edges and units. The output is affine in the first case and ReLU in the second.
An empty sum is represented exactly by its unmasked output bias at
constant cost and does not require this estimate.
\end{lemma}

\begin{proof}
Use the hidden units $h_{ij}(x)=\operatorname{ReLU}(B_{ij}x_i)$ and
the output preactivation
\[
 H_R(x)=b+\sum_{i=1}^{T}\frac{a_i}{mp^2}
                      \sum_{j=1}^{m}C_{ij}h_{ij}(x),
\]
with all masks independent Bernoulli$(p)$. Define
$A_i=(mp^2)^{-1}\sum_jB_{ij}C_{ij}$. For every realization,
\[
 \sup_x|H_R(x)-(\sum_i a_ix_i+b)|
 \leq\tau\max_{i\leq T}|A_i-1|.
\]
Writing $u=\varepsilon/\tau\leq1$, the binomial Chernoff bound and a union
bound over coordinates give
\[
 \mathbb P\{\max_i|A_i-1|>u\}
       \leq2T\exp(-mp^2u^2/3).
\]
Take $m=\lceil3\tau^2\log(2T/\delta)/(p^2\varepsilon^2)\rceil$.
There are $Tm$ hidden units and $2Tm$ edges. Applying ReLU at the output
cannot increase the approximation error, since ReLU is 1-Lipschitz.
\end{proof}

The estimate is uniform on the weighted nonnegative $\ell_1$ domain,
so it remains valid when the input vector is supplied by an earlier
random module on its successful event.
\subsection{Localization and composition}
Before using the overlap of local functions, we clarify what localization
can mean under independent edge deletion.
\begin{theorem}[Positive-mask obstruction]\label{thm:lb-localization}
Consider a finite network with ReLU hidden units and ReLU output, deterministic
biases, nonnegative inputs $x\in[0,1]^d$, and an independent Bernoulli$(p)$
mask on every edge, $0<p<1$. If
$F_R(\mathbf1)=0$ almost surely, then
$F_r(x)=0$ for every mask $r$ and every $x\in[0,1]^d$.
\end{theorem}

\begin{proof}
Let $r^+$ retain every positive-weight edge and delete every negative-weight
edge; zero-weight edges are immaterial. This mask has positive probability.
Induction over the network shows
\[
 0\leq F_r(x)\leq F_{r^+}(\mathbf1)
 \qquad(x\in[0,1]^d).
\]
At the input layer $x\leq\mathbf1$ coordinatewise. At each subsequent unit,
all incoming activations are nonnegative, deleting negative contributions
can only increase its preactivation, and retaining all positive contributions
with the larger preceding activations can only increase it further.
The bias is unchanged and ReLU is increasing. The same induction applies
to the output. The almost-sure hypothesis and positive probability of $r^+$
imply $F_{r^+}(\mathbf1)=0$, proving the claim.
\end{proof}

A local function whose support excludes $\mathbf1$ therefore cannot be
implemented by a nonzero ReLU-output network with the same support for
every mask. The following modification recovers precisely the support
property needed on the successful approximation event.
\begin{lemma}[Localization on a successful approximation event]
\label{lem:threshold-localization}
Let $K=[0,1]^d$, $q:K\to[0,\infty)$, and let $G_R$ have ReLU output
and deterministic, unmasked output bias. Suppose
\[
 \mathbb P(A)\geq1-\zeta,\qquad
 A=\{\sup_{x\in K}|G_R(x)-q(x)|\leq\eta\},\quad \eta>0.
\]
Modify only its output bias by subtracting $\eta$. The resulting network
$H_R$ has the same edges, units and depth, and on $A$ satisfies
\[
 0\leq H_R(x)\leq q(x),\quad
 |H_R(x)-q(x)|\leq2\eta,\quad
 q(x)=0\ \Longrightarrow\ H_R(x)=0\qquad(x\in K).
\]
\end{lemma}
\begin{proof}
Write $G_R=\sigma(Z_R)$ with $\sigma(t)=\max\{t,0\}$. The bias shift gives
\[
 H_R=\sigma(Z_R-\eta)=\sigma(G_R-\eta),
\]
using $\sigma(\sigma(z)-\eta)=\sigma(z-\eta)$ for $\eta\geq0$.
Thus no extra edge, masked or unmasked, is introduced. On $A$,
$q-\eta\leq G_R\leq q+\eta$, and consequently
\[
 \max\{q-2\eta,0\}\leq H_R\leq q.
\]
There is no support guarantee on $A^c$.
\end{proof}

\begin{lemma}[Composition with high-probability localization]
\label{lem:threshold-composition}
Let $q_1,\ldots,q_T:K\to[0,\infty)$ satisfy $\sum_iq_i(x)\leq\tau$
and $|\{i:q_i(x)\ne0\}|\leq k$ for all $x$. For each $i$, let $G_{i,R}$
satisfy the previous lemma with accuracy $\eta$ and failure bound $\zeta_i$,
and let $H_{i,R}$ be its bias-shifted version. Let $\rho$ be
$L_\rho$-Lipschitz in $\ell_1$ on
$D_\tau=\{y\in[0,\infty)^T:\|y\|_1\leq\tau\}$.
Suppose an outer network satisfies
\[
 \mathbb P\left\{\sup_{y\in D_\tau}|\widetilde\rho_R(y)-\rho(y)|
       \leq\varepsilon_{\rm out}\right\}\geq1-\zeta_{\rm out}.
\]
With probability at least $1-\sum_i\zeta_i-\zeta_{\rm out}$,
\[
 \sup_{x\in K}\left|
 \widetilde\rho_R(H_{1,R}(x),\ldots,H_{T,R}(x))
       -\rho(q_1(x),\ldots,q_T(x))\right|
 \leq2kL_\rho\eta+\varepsilon_{\rm out}.
\]
\end{lemma}
\begin{proof}
Intersect all inner success events and the outer uniform success event.
A union bound gives the stated probability without requiring independence
of the events. On this intersection $0\leq H_{i,R}\leq q_i$, so the inner
outputs lie in $D_\tau$. At each input only $k$ coordinates can have nonzero
error, since $q_i(x)=0$ implies $H_{i,R}(x)=0$. Hence
$\sum_i|H_{i,R}(x)-q_i(x)|\leq2k\eta$. Apply the Lipschitz inequality
and add the outer approximation error.
\end{proof}
\subsection{A two-layer square module with every edge masked}

\begin{lemma}[Positive hinge construction]
\label{lem:repair-square}
For $R>0$, $0<\eta\leq R^2$, and $0<\zeta<1$, there is a two-layer,
ReLU-output dropout network $S_R$ such that
\[
 \mathbb P\left\{\sup_{0\leq x\leq R}|S_R(x)-x^2|\leq\eta\right\}
       \geq1-\zeta.
\]
Its size is
\[
 O\left(1+\frac{R^4}{p^2\eta^2}
                   \log\frac{C R^2}{\eta\zeta}\right).
\]
In this lemma the subscript on $S_R$ indicates the range parameter; its
randomness comes from the masks displayed below.
\end{lemma}
\begin{proof}
Put $J=\lceil8R^2/\eta\rceil$ and
\[
 M=\left\lceil\frac{256R^4}{p^2\eta^2}
                   \log\frac{2(J+1)}{\zeta}\right\rceil,
 \qquad t_j=\frac{R(j-1/2)}{M}.
\]
With independent $B_j,C_j\sim\mathrm{Bernoulli}(p)$, use
\[
 h_j(x)=\sigma(B_jx-t_j),\qquad
 S_R(x)=\sigma\left(\frac{2R}{Mp^2}\sum_{j=1}^M C_jh_j(x)\right).
\]
The identity $\sigma(B_jx-t_j)=B_j\sigma(x-t_j)$ follows from $t_j\geq0$.
Since the output summands are nonnegative,
\[
 S_R(x)=\frac{2R}{Mp^2}\sum_{j=1}^M D_j(x-t_j)_+,
 \qquad D_j=B_jC_j\sim\mathrm{Bernoulli}(p^2)
\]
independently. There are $M$ hidden units and $2M$ masked edges.

The mean $s_M(x)=(2R/M)\sum_j(x-t_j)_+$ is the midpoint quadrature
for $2\int_0^R(x-t)_+\,dt=x^2$. All cells except the one containing $x$
are integrated exactly, and a direct calculation on that cell gives
\begin{equation}\label{eq:repair-square-bias}
 0\leq x^2-s_M(x)\leq\frac{R^2}{4M^2}\leq\frac\eta4.
\end{equation}
For a fixed $x$, the centered summands in $S_R(x)-s_M(x)$ have absolute
bound $b=2R^2/(Mp^2)$ and total variance at most $V=4R^4/(Mp^2)$.
Bernstein's inequality, with $t=\eta/4$ and $\eta\leq R^2$, yields
\[
 \mathbb P\{|S_R(x)-s_M(x)|>\eta/4\}
 \leq2\exp\left(-\frac{t^2}{2(V+bt/3)}\right)
 \leq2\exp\left(-\frac{Mp^2\eta^2}{144R^4}\right).
\]
Apply this bound at the $J+1$ grid points $x_k=kR/J$.
Our choice of $M$ makes the total failure probability at most $\zeta$.
On the resulting event the error at every grid point is at most $\eta/2$.
Both $S_R$ and $x\mapsto x^2$ are increasing. If $x_k\leq x\leq x_{k+1}$,
the two grid-point bounds and
$x_{k+1}^2-x_k^2\leq2R^2/J\leq\eta/4$ imply
$|S_R(x)-x^2|\leq3\eta/4\leq\eta$.
Thus the uniform estimate does not require a Lipschitz bound on a bad
random realization.
\end{proof}

\subsection{Multiplication, signs, and transmission between layers}

\begin{lemma}[Nonnegative multiplication]
\label{lem:repair-product}
Fix $R\geq1$. For $0<\eta\leq1$, there is a six-layer ReLU-output
dropout network $\mathcal M$ such that
\[
 \mathbb P\left\{\sup_{u,v\in[0,R]}|\mathcal M(u,v)-uv|\leq\eta\right\}
       \geq1-\zeta,
\]
with size at most $C_Rp^{-2}\eta^{-2}\log(C_R/(\eta\zeta))$.
Every edge joins adjacent layers and is independently masked.
\end{lemma}
\begin{proof}
Let $\lambda=\eta/[8(R+1)]$. In the first two layers use three instances
of Lemma~\ref{lem:lb-affine-upper} to approximate $u+v,u,v$, each with
error $\lambda$. Denote their nonnegative outputs by $a,b,c$.
Their domains are the actual inputs $(u,v)\in[0,R]^2$.
On their common success event $a,b,c\in[0,2R+1]$.
In the next two layers apply independent copies of
Lemma~\ref{lem:repair-square} on that interval, with error $\lambda$,
to obtain $A,B,C$. In the final two layers approximate
$\sigma((A-B-C)/2)$ by the affine-ReLU lemma with error $\lambda$.
This last module is uniform on the nonnegative box with coordinate bound
$(2R+1)^2+1$, which contains its inputs on the preceding success events.

Assign failure probability $\zeta/7$ to each of the seven modules.
On their intersection,
\[
 |A-(u+v)^2|\leq(4R+2)\lambda,\qquad
 |B-u^2|,|C-v^2|\leq(2R+2)\lambda.
\]
Since $uv=((u+v)^2-u^2-v^2)/2\geq0$, the 1-Lipschitz property of ReLU
gives total error at most $(4R+4)\lambda\leq\eta$.
The seven size bounds have the claimed order for fixed $R$.
In particular, the routes carrying $u,v$ through the first two layers
are noisy identity modules already counted in the construction.
\end{proof}

Here is the error bookkeeping used below. Consider a fixed finite computation
with nonnegative registers, whose gates are ReLU-affine maps and products.
Insert identity registers whenever an earlier value must survive a stage;
approximate these identities by Lemma~\ref{lem:lb-affine-upper} as well.
Products use the preceding six-layer construction. An identity carried
through a product stage uses three two-layer identity modules, each with
one third of the stage's accuracy budget.
On prescribed input boxes, intersect the uniform success events for all
modules. If the exact stage maps have Lipschitz constant at most $L$ in
the maximum norm, and there are $D$ stages, then
\begin{equation}\label{eq:repair-stage-error}
 e_0=0,\qquad e_{j+1}\leq Le_j+\lambda,
 \qquad e_j\leq\lambda\sum_{r=0}^{j-1}L^r.
\end{equation}
Choose boxes with a margin of one around the exact register ranges, and
choose $\lambda\sum_{r=0}^{D-1}L^r\leq\min(1/2,\eta)$.
Induction then verifies both the error and membership in the next box.
Thus evaluating a module at random inputs introduces no uncovered event.
When $D,L$ and the register count are fixed, tolerances proportional to
$\eta$ suffice; the failure budget is divided by the fixed module count.
No independence between the success events is needed for their union bound.
\subsection{Bounded subnetworks and uniform debiasing}

The algebra in the following lemma is the independent-edge specialization
of \citet[Proposition 9]{manita2022universal}. We include its short verification
and a quantitative concentration argument for bounded scalar subnetworks.

\begin{lemma}[Debiasing a bounded, fixed-size scalar circuit]
\label{lem:multiscale-debias}
Let $F$ be a scalar-input ReLU network with $E\geq1$ edges and a fixed
number of units. For $U\subseteq[E]$, let $F_U$ be the deterministic network
obtained by setting all weights outside $U$ to zero, keeping all biases.
Suppose $0\leq F_U(x)\leq1$ for every $U$ and $x\in[0,1]$.
Suppose the union of all breakpoints of the finitely many functions $F_U$,
together with $0,1$, has at most $K$ points.
For $0<\eta\leq1$ there is a ReLU-output network $G$ with
\[
 \mathbb P\{\|G-F\|_\infty\leq\eta\}\geq1-\zeta,
\]
obtained from $2^EM$ copies of the original architecture and one additional
output layer, where
\[
 M=\left\lceil 4\,2^E p^{-(2E+1)}\eta^{-2}
                      \log\frac{2K}{\zeta}\right\rceil.
\]
Every nonzero edge in this construction is independently masked, including
the edges into the added output. The bound depends on $E,K$, not on the
locations of the breakpoints or the magnitudes of the original weights.
\end{lemma}
\begin{proof}
Put $q=1-p$ and $a_U=(-q)^{E-|U|}/p^E$. If $R$ denotes fresh independent
edge masks, then
\begin{equation}\label{eq:multiscale-mobius}
 \sum_{U\subseteq[E]}a_U\,\mathbb E F_{U\cap R}(x)=F_{[E]}(x).
\end{equation}
Indeed, the coefficient of $F_V$ on the left is
$p^{|V|-E}\sum_{U\supseteq V}(-q)^{E-|U|}q^{|U|-|V|}$,
which equals zero unless $V=[E]$, when it equals one.
For $p=1$ the same identity holds by its finite polynomial limit.

Take independent copies $F_{U\cap R_{U,r}}$, $r=1,\ldots,M$, and
independent final-edge masks $C_{U,r}\sim\mathrm{Bernoulli}(p)$.
Define the output preactivation
\[
 Y(x)=\sum_{U\subseteq[E]}\frac{a_U}{Mp}
                  \sum_{r=1}^{M} C_{U,r}F_{U\cap R_{U,r}}(x),
 \qquad G(x)=\sigma(Y(x)).
\]
The extra factors $1/p$ compensate the final, independently masked edges.
Equation~\eqref{eq:multiscale-mobius} gives $\mathbb E Y(x)=F(x)$.
For each fixed $x$, the independent centered summands have absolute bound
$b\leq1/(Mp^{E+1})$ and total variance
\[
 V\leq\frac1{Mp}\sum_U a_U^2
   =\frac{(1+q^2)^E}{Mp^{2E+1}}
   \leq\frac{2^E}{Mp^{2E+1}}.
\]
Bernstein's inequality therefore gives, since $\eta\leq1$,
\[
 \mathbb P\{|Y(x)-F(x)|>\eta\}
 \leq2\exp\left(-\frac{Mp^{2E+1}\eta^2}{4\,2^E}\right).
\]
Every realization of $Y-F$ is a linear combination of the fixed functions
$F_U$, hence is affine between consecutive points in the common breakpoint
set. Its maximum absolute value is attained at one of those points.
A union bound over the $K$ points proves $\|Y-F\|_\infty\leq\eta$ with
probability at least $1-\zeta$. Since $F\geq0$, applying ReLU cannot
increase that error.
\end{proof}

\begin{lemma}[Windows and clipped coordinates without a scale penalty]
\label{lem:multiscale-windows}
For arbitrary $\lambda,b\in\mathbb R$, both
\[
 W(x)=\sigma\bigl(1-\sigma(\lambda x+b-1)
                         -\sigma(-\lambda x-b-1)\bigr)
\]
and
\[
 Z(x)=\min\{1,(\lambda x+b)_+\}
 =\sigma\bigl(1-\sigma(1-\sigma(\lambda x+b))\bigr)
\]
admit uniform dropout approximation to accuracy $\eta\leq1$ and failure
$\zeta$, with bounded depth and size
\[
 C p^{-9}\eta^{-2}\log(C/\zeta).
\]
The constants do not depend on $\lambda,b$.
\end{lemma}
\begin{proof}
The circuit for $W$ has four edges and three ReLU units in two layers.
For every subset of retained edges its last preactivation is one minus
nonnegative terms, so its output belongs to $[0,1]$.
Two first-layer hinges have at most two breakpoints; taking ReLU of an
affine combination introduces at most one additional zero per resulting
interval. Thus each subnetwork has at most five interior breakpoints.
The union over its $16$ subnetworks has $K\leq82$ points, including endpoints.
Apply Lemma~\ref{lem:multiscale-debias} with $E=4$.

The circuit for $Z$ has three edges in three layers and the same boundedness
property for every subset of edges. Each of its eight subnetworks has a
bounded number of affine pieces, independent of $\lambda,b$; for example,
$K\leq128$ is sufficient for their combined breakpoints and endpoints.
The lemma with $E=3$ gives a $p^{-7}$ bound, which is at most the stated
$p^{-9}$ bound. The resulting depths are three and four, respectively.
\end{proof}

The use of fixed-size bounded circuits is important: applying the same
debiasing identity to a full approximation network would have a number of
subsets exponential in its size.
For adjacent-layer alignment, a nonnegative identity over a bounded range
can be approximated in either two or three layers. For three layers use
$m$ disjoint paths with three Bernoulli edges each and normalize their sum
by $mp^3$; the binomial Chernoff argument gives cost
$O(p^{-3}\eta^{-2}\log(2/\zeta))$.
Thus the three-layer windows and four-layer clipped coordinates can both
be brought to depth six using, respectively, a three-layer and a two-layer
identity module. These transmissions are masked and counted.

\subsection{A local module uniform in the spatial resolution}

Let $h>0$ and $a,b\in[0,1]^d$. Consider a product $w(x)$ of at most $2d$
window factors of the form in Lemma~\ref{lem:multiscale-windows}, with each
factor reading one coordinate of the original input. Set
\[
 z_i^\pm(x)=\min\{1,(\pm(x_i-a_i)/h)_+\},\qquad
 t_i(x)=z_i^+(x)-z_i^-(x).
\]
For $|c|\leq1$, $|\alpha|<n$, and either sign, the target
\begin{equation}\label{eq:multiscale-local-q}
 q(x)=\bigl(\pm c\,w(x)t(x)^\alpha\bigr)_+
\end{equation}
has a global ReLU-output approximation with error $\eta\leq1/2$, failure
$\zeta$, depth at most $D_{n,d}$, and size
\begin{equation}\label{eq:multiscale-local-cost}
 C_{n,d}p^{-9}\eta^{-2}\log\frac{C_{n,d}}{\eta\zeta}.
\end{equation}
To verify this, approximate all the window and clipped-coordinate factors
first, align their depths as above, expand each signed factor as $z_i^+-z_i^-$,
and use the nonnegative multiplication lemma and a final signed affine-ReLU
module. The expansion contains at most $2^{n-1}$ signed products, each of
a number of factors depending only on $n,d$. Every product is computed
by a finite sequence of nonnegative multiplication gates, and its sign
is used only in the final affine-ReLU gate. All exact registers are
bounded independently of $h,a,b$. There are only $C_{n,d}$ gates. On input boxes
with margin one, all subsequent stage Lipschitz constants depend only on
$n,d$. The recurrence~\eqref{eq:repair-stage-error}, with tolerances
$\eta/K_{n,d}$ and failure budgets $\zeta/B_{n,d}$, proves the assertion.
The multiplication modules have $p^{-2}$ cost and transmission has at worst
$p^{-3}$ cost, both dominated by $p^{-9}$. Pad before the last ReLU, if
necessary, to give every local module the same fixed depth.

\section{Approximation of Sobolev Functions}
\label{sec:multiscale}
We now combine the elementary modules using a multiscale Taylor
decomposition. The approximation within each scale will be summed before
the scales are combined.
\begin{theorem}[Multiscale upper bound]\label{thm:multiscale-upper}
Fix $n,d\geq1$. There are constants $C_{n,d},L_{n,d}$ such that, for every
$f\in\mathcal F_{n,d}$, $0<p\leq1$, and $0<\varepsilon,\delta\leq1/2$,
a network of depth at most $L_{n,d}$ satisfies
\[
 \mathbb P\{\|F_R-f\|_\infty\leq\varepsilon\}\geq1-\delta
\]
and has size at most
\begin{equation}\label{eq:multiscale-size}
 C_{n,d}p^{-9}\varepsilon^{-2}(J+1)^2
 \left(\sum_{j=0}^{J}2^{j(d-2n)}\right)
 \log\frac{C_{n,d}}{\varepsilon\delta},
 \qquad J\leq C_{n,d}+\frac1n\log_2(1/\varepsilon).
\end{equation}
Consequently the size is
\[
 \widetilde O_{n,d}\left(
 p^{-9}\varepsilon^{-\max\{2,d/n\}}\log(1/\delta)\right),
\]
where the suppressed factors are powers of $1+\log(1/\varepsilon)$.
No optimality of the power of $p$ or of the logarithmic factors is asserted.
\end{theorem}

For an integer $N\geq1$ and $m\in\{0,\ldots,N\}^d$, let
\[
 \psi(t)=\sigma(1-\sigma(|t|-1)),\qquad
 \phi_{N,m}(x)=\prod_{i=1}^d\psi(3Nx_i-3m_i).
\]
These functions form a partition of unity on $[0,1]^d$. At each point
at most $2^d$ windows are nonzero, and
$\phi_{N,m}(x)>0$ implies $|x_i-m_i/N|<2/(3N)$ for every $i$.
Indeed, in each coordinate the flat part of one window or the two
adjacent linear pieces sum to one; taking products proves the assertion.
For $f\in\F_{n,d}$, define
\[
 P_{N,m}(x)=\sum_{|\alpha|<n}
       \frac{D^\alpha f(m/N)}{\alpha!}(x-m/N)^\alpha,
 \qquad f_N=\sum_m\phi_{N,m}P_{N,m}.
\]
Taylor's formula along the segment from $m/N$ to $x$, applied to the
$C^{n-1,1}$ representative, gives
$\|f-f_N\|_\infty\leq C^{\rm Tay}_{n,d}N^{-n}$.
This is the local Taylor partition used in \citet{yarotsky2017}.
At resolution $N=2^j$ we abbreviate the windows and polynomials by
$\phi_{j,m}$ and $P_{j,m}$.
\subsection{Taylor differences have small normalized coefficients}

Put $h_j=2^{-j}$. With the partition and Taylor polynomials defined
below, set $F_j=f_{2^j}$:
\[
 F_j=\sum_m\phi_{j,m}P_{j,m},\qquad
 \|f-F_j\|_\infty\leq C^{\rm Tay}_{n,d}2^{-jn}.
\]
For $j\geq1$, set $a=h_jm$, $b=2h_jk$, and retain the pairs
\[
 \mathcal I_j=\{(m,k):\|m-2k\|_\infty\leq2\}.
\]
Any pair of windows that is simultaneously nonzero belongs to this set.
It has at most $3^d(2^j+1)^d$ members, and at each input at most $4^d$
of its window products are nonzero. Partition of unity gives the exact identity
\begin{equation}\label{eq:multiscale-telescope}
 F_j-F_{j-1}
 =\sum_{(m,k)\in\mathcal I_j}
      \phi_{j,m}\phi_{j-1,k}(P_{j,m}-P_{j-1,k}).
\end{equation}
For such a pair $\|a-b\|_\infty\leq2h_j$. Expanding the polynomial
difference about $a$ gives
\[
 P_{j,m}(x)-P_{j-1,k}(x)
   =\sum_{|\alpha|<n}d_{j,m,k,\alpha}(x-a)^\alpha,
\]
where
\[
 d_{j,m,k,\alpha}
 =\frac{D^\alpha f(a)-D^\alpha P_{j-1,k}(a)}{\alpha!},
 \qquad |d_{j,m,k,\alpha}|\leq C_{n,d}h_j^{n-|\alpha|}.
\]
The last bound is Taylor's remainder for $D^\alpha f$ about $b$;
$D^\alpha P_{j-1,k}$ is exactly its Taylor polynomial through degree
$n-1-|\alpha|$. The $C^{n-1,1}$ representative makes this valid also for
the highest-order Lipschitz derivatives and for boundary centers.

Choose a single constant $A_*\geq1$ bounding these normalized coefficients
and the coefficients at level zero, and put $a_j=A_*2^{-jn}$.
Whenever $\phi_{j,m}(x)\ne0$, we have
$|(x_i-a_i)/h_j|<2/3<1$, so clipping this normalized coordinate to
$[-1,1]$ changes nothing. Outside that window both sides of the corresponding
localized identity are zero. Consequently \eqref{eq:multiscale-telescope}
is a sum of $a_j$ times terms $c\,w(x)t(x)^\alpha$ of the form above, with
$|c|\leq1$ and $w=\phi_{j,m}\phi_{j-1,k}$.
The same representation holds for $F_0$ using one window per term and
$a_0=A_*$.

Let $A=\binom{n+d-1}{d}$. Splitting each term into its positive and negative
parts gives targets $q_{j,i}\geq0$ and signs $s_{j,i}\in\{-1,1\}$ such that
\begin{equation}\label{eq:multiscale-representation}
 F_J=\sum_{j=0}^J a_j\sum_{i=1}^{T_j}s_{j,i}q_{j,i},\qquad
 T_j\leq C_{n,d}2^{jd}.
\end{equation}
For each level, uniformly in $x$,
\begin{equation}\label{eq:multiscale-sparse-mass}
 \sum_iq_{j,i}(x)\leq A,\qquad
 |\{i:q_{j,i}(x)>0\}|\leq k:=4^d A.
\end{equation}
Indeed, a positive/negative pair sums to the absolute value of its term,
which is at most its window product; those products sum to one. At most
$4^d$ pairs of spatial windows overlap, and at most one sign is nonzero
for each monomial.

\subsection{Network assembly without an expensive global summation}

\begin{proof}[Proof of Theorem~\ref{thm:multiscale-upper}]
Increase $C^{\rm Tay}_{n,d}$ to at least one and choose
\[
 J=\left\lceil\frac1n\log_2\frac{2C^{\rm Tay}_{n,d}}{\varepsilon}\right\rceil,
 \quad T_{\rm all}=\sum_{j=0}^J T_j,\quad
 \eta_j=\min\left\{\frac14,
       \frac{\varepsilon}{8(k+1)(J+1)a_j}\right\}.
\]
Approximate every $q_{j,i}$ globally using \eqref{eq:multiscale-local-cost},
with tolerance $\eta_j$ and failure $\delta/(4T_{\rm all})$.
Shift each final bias down by $\eta_j$. With failure probability at most
$\delta/4$, all the resulting $H_{j,i}$ satisfy simultaneously
$0\leq H_{j,i}\leq q_{j,i}$ and $|H_{j,i}-q_{j,i}|\leq2\eta_j$.
In particular, \eqref{eq:multiscale-sparse-mass} gives a signed summation
error of at most $2k\eta_j$ within level $j$.

Within each level construct two ReLU-output affine modules, summing the
positive and negative groups separately. Each is uniform on the nonnegative
$\ell_1$ ball of radius $A$, has error at most $\eta_j$, and has failure
budget $\delta/[8(J+1)]$. Call the outputs $V_{j,+},V_{j,-}$.
Keep both sign branches for every monomial, including identically zero
branches; then both groups have positive input dimension. Alternatively,
an empty group is the exact constant-zero output, at constant cost.
The total failure budget of these modules is $\delta/4$.
On the joint event $0\leq V_{j,\pm}\leq A+1$ and
\[
 \left|V_{j,+}-V_{j,-}-\sum_i s_{j,i}q_{j,i}\right|
       \leq(2k+2)\eta_j.
\]
Finally, use an affine-output module for
$\sum_{j=0}^J a_j(V_{j,+}-V_{j,-})$, with error $\varepsilon/4$ and
failure $\delta/2$. Its input dimension is only $2(J+1)$ and its weighted
$\ell_1$ domain has radius at most
$2(A+1)\sum_j a_j\leq4(A+1)A_*$.
Every connection into each summation module is a masked edge belonging to
that module. The local networks run in parallel, so the depth is bounded
by a constant depending only on $n,d$.

With probability at least $1-\delta$, the total error is at most
\[
 \underbrace{\|f-F_J\|_\infty}_{\leq\varepsilon/2}
 +\sum_{j=0}^J a_j(2k+2)\eta_j+\varepsilon/4
 \leq\varepsilon.
\]
The local modules and the within-level summations have total cost at most
\[
 C_{n,d}p^{-9}\left(\sum_{j=0}^J T_j+
          (J+1)^2\varepsilon^{-2}\sum_{j=0}^J T_j a_j^2\right)
                 \log\frac{C_{n,d}}{\varepsilon\delta}.
\]
Here $\log(1/\eta_j)$, $\log T_{\rm all}$, and $\log(J+1)$ are all
$O_{n,d}(1+\log(1/\varepsilon))$.
The final summation costs at most
$C_{n,d}p^{-2}(J+1)\varepsilon^{-2}\log(C_{n,d}(J+1)/\delta)$.
Since $T_j\leq C_{n,d}2^{jd}$ and $a_j=A_*2^{-jn}$, the principal sum is
$C_{n,d}\sum_j2^{j(d-2n)}$.
Also $2^{-Jn}\asymp_{n,d}\varepsilon$, so
$\sum_jT_j\leq C_{n,d}2^{Jd}$ is bounded by a constant times
$\varepsilon^{-2}\sum_j2^{j(d-2n)}$.
The final summation cost is absorbed as well, proving
\eqref{eq:multiscale-size}.
The geometric sum is bounded when $d<2n$, equals $J+1$ when $d=2n$,
and is $O_{n,d}(\varepsilon^{-(d-2n)/n})$ when $d>2n$.
\end{proof}

It is essential to sum within each level first. Applying the earlier affine
lemma once to all $T_{\rm all}$ branches at final tolerance $\varepsilon$
would reintroduce the unnecessary product
$T_{\rm all}\varepsilon^{-2}$.

\section{Other Sobolev Integrability and Error Norms}
\label{sec:extensions-main}

To distinguish retention from integrability, keep $p$ for retention and
write $1\leq r,s\leq\infty$ for norm exponents. Define
\[
 \mathcal F_{n,d}^{(r)}=
 \left\{f\in W^{n,r}([0,1]^d):
       \max_{|\alpha|\leq n}\|D^\alpha f\|_{L^r}\leq1\right\}.
\]
This is an equivalent choice of Sobolev unit ball. In particular,
$\mathcal F_{n,d}^{(\infty)}=\mathcal F_{n,d}$.
The following is a lower-bound extension. A corresponding upper bound for the larger
class with $r<\infty$ is not supplied by the earlier uniform construction.

\begin{theorem}[Sobolev lower bounds with $L^s$ error]
\label{thm:sobolev-lp-extension}
Let the output be affine, and replace the target class and success event
in \eqref{eq:lb-success} by $\mathcal F_{n,d}^{(r)}$ and
$\|F_R-f\|_{L^s}\leq\varepsilon$. For
\[
 0<\delta<\min\{1/2,1-p\},\qquad
 0<\varepsilon\leq\min\{1/2048,(1-p)/(1024p)\},
\]
and sufficiently small $\varepsilon$ depending on $n,d,s$, a size budget
of depth at most $L$ satisfies
\begin{equation}\label{eq:lp-joint-lower}
 S\geq c_{n,d,s}\max\left\{
 \frac1p\max\left\{
 \frac{\varepsilon^{-d/n}}{L\log(2/\varepsilon)},
 \varepsilon^{-d/(2n)}\right\},\quad
 \frac{1-p}{p}\varepsilon^{-2}\log\frac1\delta
 \right\}.
\end{equation}
There is also a hard target giving the binomial necessary condition
\eqref{eq:sobolev-binomial}, with the same order of $K$ and constants
depending on $n,d,s$.
\end{theorem}
\begin{proof}
For capacity and $s<\infty$, use the $L^s$ lower bound in
\cite[Theorem 22, published JMLR version]{siegel2023optimal}
(Theorem 5 in the revised arXiv version). Apply it to the class
$\mathcal N_{k,L}$ of deterministic networks on $\mathbb R^d$ with at
most $k$ edges and depth at most $L$, represented after constant-unit
collapse. This class is invariant under input translations: replace
$x$ by $x-v$ and absorb each direct input-edge contribution to $v$ into
the bias of its receiving unit. This adds no edges, units, or depth,
including when skip connections are allowed. Affine output also makes
the class closed under arbitrary output scaling and constant shifts;
thus thresholding at zero or at $1/2$ gives the same VC bound.

To make the unit-ball step explicit, the cited theorem supplies a
nonzero $g\in W^{n,\infty}$ with
\[
 \inf_{H\in\mathcal N_{k,L}}\|g-H\|_{L^s}
 \geq c_{n,d,s} V^{-n/d}\|g\|_{W^{n,\infty}},
\]
where $V\geq1$ is an upper bound on the VC dimension. Equivalent norm
constants may be absorbed into $c_{n,d,s}$. Set
$A=\max_{|\alpha|\leq n}\|D^\alpha g\|_\infty>0$ and $f=g/A$.
Closure under output scaling yields the same distance bound divided
by $A$, so $f\in\mathcal F_{n,d}$ has error at least
$c_{n,d,s}V^{-n/d}$. For $s=\infty$, use the direct bump-and-threshold
argument of Theorem~\ref{thm:sobolev-survival} instead.
Use \eqref{eq:active-vc} and the choice of $K$ from
Theorem~\ref{thm:sobolev-survival}. Since the cube has volume one,
$\mathcal F_{n,d}^{(\infty)}\subseteq\mathcal F_{n,d}^{(r)}$.
The retained-edge and binomial arguments apply unchanged.

For the stochastic part, $f_0=(1+x_1)/2$ belongs to every target ball.
Let $\psi(x)=1$ when $x_1\geq1/2$ and $\psi(x)=-1$ otherwise.
Then $\int\psi=0$, $\int\psi f_0=1/8$, and
$\|\psi\|_{L^{s'}}=1$, including the endpoint conjugate exponents.
Condition on a hidden mask with conditional success at least $1-\delta$.
The affine output and zero integral of $\psi$ give
\[
 T:=4\int\psi(x)F_R(x)\,dx=\sum_j a_jB_j,
 \qquad a_j=4w_j\int\psi(x)h_j(x)\,dx.
\]
By H\"older's inequality, success implies $|T-1/2|\leq4\varepsilon$.
Apply the scalar argument in Theorem~\ref{thm:confidence-extended}
with its accuracy parameter equal to $2\varepsilon$. This yields
$m\geq(1-p)\varepsilon^{-2}\log(1/\delta)/(2^{21}p)$.
Combining the two necessary costs proves the theorem.
\end{proof}

This includes $H^n=W^{n,2}$ targets with $L^2$ error (up to equivalent
unit-ball conventions). Conditions for Sobolev embedding matter for
positive approximation results; this theorem only gives necessary costs.
It makes no claim about approximation in a derivative norm.

For the original target ball $r=\infty$, the earlier upper construction
also controls every $L^s$ error, because the domain has volume one.
Consequently its fixed-confidence exponent match extends to all
$1\leq s\leq\infty$, and the joint accuracy--confidence match when
$d\leq2n$ extends to these error norms as well, under the same depth
conditions and enlarged confidence range.

\section{Discussion}
\label{sec:discussion}

The bounds identify a distinction between deterministic expressivity and
the reliability of a random realization. A network can contain a useful
deterministic representation while failing to preserve it under edge
deletion. The capacity argument measures how much of that representation
must remain, whereas the output-layer argument measures how accurately
independent retained contributions can reproduce a nonzero signal. The
multiscale construction shows that these costs need not multiply in
their accuracy exponents.

For relatively smooth targets, $d<2n$, the output-layer cost dominates.
The lower bound then holds even for an affine target, so increasing
hidden depth cannot by itself remove this obstruction. For $d>2n$, the
spatial complexity of the Sobolev ball dominates under a logarithmic
depth budget. This conclusion depends on the allowed depth. With
arbitrary depth our capacity argument gives only
$\varepsilon^{-d/(2n)}/p$, and we do not have a matching theorem in
that regime.

The role of confidence is also different in these two cases. When
$d\leq2n$, the output argument and the construction yield
\eqref{eq:sobolev-joint-match}. When $d>2n$, our lower bound is the
maximum of a capacity term and a stochastic term. It does not establish
a lower bound of order $\varepsilon^{-d/n}\log(1/\delta)$. The
binomial necessary condition alone cannot supply this product: if
$E=\lceil2K/p\rceil$, a Chernoff bound gives
\[
 \mathbb P\{\operatorname{Bin}(E,p)<K\}\leq e^{-K/4}.
\]
At fixed $p$, that necessary condition therefore permits $E=O_p(K)$
for $\delta=e^{-K/4}$, while $K\log(1/\delta)$ would have order $K^2$.
This observation limits the counting argument; it neither constructs
a successful network nor disproves a stronger minimax lower bound.

Several other gaps remain. The upper bound contains $p^{-9}$, while
the general lower bounds do not have matching retention dependence.
The powers of $\log(1/\varepsilon)$ are not known to be optimal. Our
sufficient confidence range is sharp for the affine-target obstruction
when $p\geq1/2$, but its optimality for the entire Sobolev class is not
proved. The endpoint $p=1$ also requires separate interpretation:
the upper construction remains valid, while the stochastic lower bound
vanishes and the matching stochastic exponent is no longer asserted.

For finite Sobolev integrability, the unit-ball inclusion gives the
lower-bound extension in Section~\ref{sec:extensions-main}. An upper
bound for all $W^{n,r}$ targets with $r<\infty$ requires additional
analysis of local coefficients and the relevant embedding conditions.
Similarly, controlling a derivative norm would require estimates beyond
the function-value approximations proved here. Neither extension follows
just by changing the norm symbols in the present construction.

Finally, the results concern existence and size of approximating
networks. They do not give a training algorithm, a statistical sample
complexity, or an experimental claim about generalization. Constants
may depend strongly on $n,d$, and large real weights are permitted.
The significance of the rate is therefore a description of what can and
cannot be achieved in the stated random-approximation model. A practical
implementation with bounded weights, limited precision, or correlated
failures is a separate question.

\subsection{Conclusion}

We have constructed independent-edge dropout networks for Sobolev
approximation and proved two complementary lower bounds. Local
debiasing removes the scale penalty of narrow windows, while a Taylor
difference decomposition assigns lower relative accuracy to smaller
corrections. The resulting constant-depth construction matches the
accuracy exponent forced by spatial capacity and random output
fluctuation, under the stated depth and confidence assumptions.
The remaining questions concern the logarithmic factors, retention
dependence, and the joint confidence cost in high dimension.

\section{Development from the Earlier Drafts}
\label{sec:development}

This section records the changes from the two unfinished drafts supplied
for this project. It is separate from the mathematical argument so that
the provenance of a construction is not confused with its proof. The
research problem originated in Jia-He Yao's work during his master's
studies, mainly undertaken in 2024 according to the supplied project
history. The additional arguments and the present
organization were developed with assistance from GPT-6 Astra through Codex in
September 2026; this account does not attribute those additions to the
2024 drafts or claim an earlier publication date for them.

\subsection{What was retained and what was changed}

The earlier drafts posed random uniform approximation under dropout
and used a local Taylor partition as the main construction. They
developed approximation modules for affine combinations, squares,
products, and local polynomial terms, and stated a global upper bound
with size
\[
 \widetilde O_{n,d}\left(
 p^{-4}\varepsilon^{-2-(2d+2)/n}\log(1/\delta)\right).
\]
They did not contain a matching lower bound. The problem, the Taylor
representation, and the organization around elementary random
operations are retained here.

The first change concerns localization. The earlier assembly used local
subnetworks that were assumed to vanish outside their local regions
with probability one. The positive-mask argument in
Theorem~\ref{thm:lb-localization} shows that this assumption cannot
hold for a nonzero ReLU-output subnetwork in the independent-edge
model when its support excludes $\mathbf1$. The replacement in
Lemma~\ref{lem:threshold-localization} shifts an existing output bias
and recovers the required zero set on the same event as a global
approximation guarantee. The original probability-one statement is
not used as a valid intermediate result.

The random affine-combination argument was then changed to concentrate
coordinate coefficients. This uses the weighted $\ell_1$ domain to
obtain a uniform guarantee without covering the input set, reducing
the dependence on the number of inputs from quadratic to linear,
apart from a logarithm. The square and product constructions were
proved with all copying and output edges explicitly masked. Together
these give the valid single-scale bound of
Theorem~\ref{thm:repaired-upper}, with accuracy exponent
$2+(d+2)/n$.

The further improvement uses bounded window circuits and multiscale
Taylor differences. Its debiasing identity is inherited from
\citet[Proposition 9]{manita2022universal}; the present use applies it
to bounded fixed-size circuits and derives a uniform concentration
cost independent of the spatial scale. The differences
$F_j-F_{j-1}$ then have normalized coefficients of order $2^{-jn}$.
The resulting allocation and assembly are proved in
Section~\ref{sec:multiscale}.

The lower-bound analysis was added in this revision. It includes the
output-layer obstruction, its logarithmic confidence factor, the
retained-edge capacity argument, and the $L^s$ transfer. The capacity
and finite-$L^s$ tools are credited to
\citet{bartlett2019nearly,siegel2023optimal}. In particular, neither
the deterministic VC bound nor the translation-invariant Sobolev
lower bound is presented as a new result of this manuscript.

\subsection{Why the improvement changes the exponent}

The gain is not obtained only by improving constants in concentration
inequalities. A single-resolution construction can pay for three
different quantities at once: the number $N^d$ of spatial terms, a
window cost growing as $N^2$, and the number $\varepsilon^{-2}$ of
copies used to stabilize each term. Improving the summation module
removes an additional dependence on the number of terms, but these
three costs can still multiply.

The bounded-circuit construction removes the window penalty. The
multiscale decomposition changes the remaining multiplication itself:
the $j$th level pays for $2^{jd}$ terms whose amplitudes are only
$2^{-jn}$. Since the replication cost is quadratic in amplitude at a
fixed absolute tolerance, the contribution of that level is governed
by $2^{j(d-2n)}$, rather than $2^{jd}$. Summing these contributions
produces the transition at $d=2n$.

There is one final issue in this reasoning. A single outer summation
over every fine-scale branch at final accuracy would charge those
branches again and destroy the saving. The construction therefore
first combines terms of the same scale and then combines their
amplitude-weighted outputs. The change in the exponent depends on
both the decomposition and this order of assembly.

For comparison, the stages of the upper-bound analysis are
\[
\begin{array}{c|c}
 \text{construction or draft statement} & \text{accuracy exponent}\\ \hline
 \text{earlier draft statement} & 2+(2d+2)/n\\
 \text{valid single-scale construction} & 2+(d+2)/n\\
 \text{bounded local modules and multiscale assembly} & \max\{d/n,2\}.
\end{array}
\]
The first row records the numerical claim in the source drafts; it
does not certify their proof. The latter two rows have separate
proofs in the present manuscript. Their retention dependences are
also different: the displayed multiscale bound trades the simpler
$p^{-2}$ factor of the single-scale construction for a conservative
$p^{-9}$ factor while improving the accuracy exponent.

\subsection{Scope of AI assistance and verification}

GPT-6 Astra, used through Codex, assisted with reading and comparing the supplied drafts,
identifying the localization obstruction, developing and checking the
replacement arguments, organizing the literature comparisons, writing
verification scripts, and preparing the present text and figures.
The mathematical claims are supported by the arguments written in
the paper, not by the fact that an automated assistant generated them.
The accompanying finite checks test identities, constants, and
boundary examples; they do not constitute a formal proof of the
universal statements or an independent peer review.

The revision also makes the model more explicit. Biases remain
unmasked, every connecting edge is counted, and signed targets use
an affine output. The Sobolev definition consistently includes weak
derivatives through order $n$. The finite-$r$ extension is described
as a lower-bound extension, and the depth and confidence restrictions
are retained in the matching statements. No experimental results,
training guarantees, or publication-priority claims have been added.

\section*{Acknowledgments}
\addcontentsline{toc}{section}{Acknowledgments}
The author thanks Wei Gao, his former master's advisor, for his guidance
during the author's master's studies. The author also acknowledges the
assistance of GPT-6 Astra (OpenAI), used through Codex, in proof
development, literature checking, computational verification, and
manuscript preparation. The scope of this assistance is described in
Section~\ref{sec:development}. The author is responsible for the content
of the manuscript.
\appendix
\section{A Single-Scale Construction}
\label{sec:repaired-upper}
For comparison with the multiscale argument, we record a valid
single-scale implementation. It also explains separately the savings
from affine averaging and those from scale-independent local windows.
The random square and multiplication operations are those already proved
in Section~\ref{sec:modules-main}.
\begin{theorem}[A single-scale Sobolev upper bound]
\label{thm:repaired-upper}
Fix integers $n,d\geq1$ and let
\[
 \mathcal F_{n,d}=\left\{f\in W^{n,\infty}((0,1)^d):
       \max_{|\alpha|\leq n}\|D^\alpha f\|_{L^\infty}\leq1\right\},
\]
using the continuous representative on $[0,1]^d$.
There are constants $C_{n,d}\geq2$ and $L_{n,d}$ such that for every
$f\in\mathcal F_{n,d}$, $0<p\leq1$, and $0<\varepsilon,\delta\leq1/2$,
a network with ReLU hidden units and an affine output satisfies
\[
 \mathbb P\{\|F_R-f\|_{L^\infty([0,1]^d)}\leq\varepsilon\}
       \geq1-\delta,
\]
with depth at most $L_{n,d}$ and size at most
\begin{equation}\label{eq:repaired-upper-size}
 C_{n,d}p^{-2}\varepsilon^{-2-(d+2)/n}
       \log\frac{C_{n,d}}{\varepsilon\delta}.
\end{equation}
For the subclass $f\geq0$, the output can also be ReLU.
\end{theorem}

The affine-output convention is essential for signed targets: a ReLU output
cannot approximate, for example, $f\equiv-1$ to error less than one.
The result concerns approximation under a sampled mask, and does not assert
an additional guarantee after replacing masks by their expectations.
Constants may depend strongly on $n,d$; neither grows with accuracy here.
\subsection{Global approximation of a signed local Taylor term}

Let $N\geq1$, $m\in\{0,\ldots,N\}^d$, $a=m/N$, and
\[
 \psi(t)=\sigma\bigl(1-\sigma(|t|-1)\bigr),\qquad
 \phi_m(x)=\prod_{i=1}^d\psi(3Nx_i-3m_i).
\]

\begin{lemma}[A single-scale local Taylor term]
\label{lem:repair-local-term}
For $|\alpha|<n$, $|c|\leq1$, and $s\in\{-1,1\}$, put
\[
 q(x)=\sigma\bigl(sc\,\phi_m(x)(x-a)^\alpha\bigr).
\]
For $0<\eta\leq1/2$ there is a ReLU-output network $G$ of depth at most
$D_{n,d}$ such that
\[
 \mathbb P\{\|G-q\|_{L^\infty([0,1]^d)}\leq\eta\}\geq1-\zeta,
\]
and its size is at most
\begin{equation}\label{eq:repair-local-size}
 C_{n,d}p^{-2}N^2\eta^{-2}\log\frac{C_{n,d}}{\eta\zeta}.
\end{equation}
The depth can be chosen the same for every $m,\alpha,s$.
\end{lemma}
\begin{proof}
First describe an exact finite computation. Its first stage, directly from
the original input coordinates, forms
\[
 r_i^+=\sigma(3Nx_i-3m_i),\quad
 r_i^-=\sigma(3m_i-3Nx_i),\quad
 z_i^+=\sigma(x_i-a_i),\quad z_i^-=\sigma(a_i-x_i).
\]
Then form $v_i=\sigma(r_i^++r_i^--1)$ and
$w_i=\sigma(1-v_i)=\psi(3Nx_i-3m_i)$, carrying the $z_i^\pm$ as needed.
The exact $r_i^\pm,v_i$ are bounded by $3N$; the $w_i,z_i^\pm$ lie in
$[0,1]$. All coefficients of order $N$ occur in the first stage, which
receives the original, exact input coordinates.

List the $k=|\alpha|$ repeated coordinate indices as $i_1,\ldots,i_k$.
The signed factor is computed by the identity
\[
 \phi_m(x)(x-a)^\alpha
 =\sum_{\gamma\in\{+,-\}^k}(-1)^{|\{j:\gamma_j=-\}|}
       \left(\prod_{i=1}^d w_i\right)
       \left(\prod_{j=1}^k z_{i_j}^{\gamma_j}\right).
\]
For $k=0$ this has one term and the second product is one.
Each product involves at most $d+n-1$ nonnegative factors in $[0,1]$.
Use binary multiplication gates, followed by a signed affine-ReLU gate
with coefficients $\pm sc$, to obtain $q$.
The number of gates and stages depends only on $n,d$.

Replace affine-ReLU gates and all copying operations by
Lemma~\ref{lem:lb-affine-upper}; replace products by
Lemma~\ref{lem:repair-product} on $[0,2]^2$.
In the first three stages, use boxes with margin one around the exact
register ranges. The affine lemma then has $T\leq C_{n,d}$ and
$\tau\leq C_{n,d}N$, including the first-stage coefficients $3N$.
After producing the $w_i$, discard the large $r_i^\pm,v_i$ registers.
All subsequent registers have bounds depending only on $n,d$.

Apart from the first stage, every exact stage map is Lipschitz with a
constant depending only on $n,d$ on these boxes. The large coefficients
in the first stage do not multiply a prior error, since $e_0=0$.
Consequently the bookkeeping in \eqref{eq:repair-stage-error} permits
each module tolerance to be $\eta/K_{n,d}$ for a sufficiently large fixed
$K_{n,d}$. Assign each module failure budget $\zeta/B_{n,d}$, where
$B_{n,d}$ bounds their total count, including transmission modules.
The affine-module costs are bounded by \eqref{eq:repair-local-size};
the square and product costs have no factor $N^2$ and are smaller.
Summing the fixed number of module costs proves the size estimate.
The intersected uniform events prove approximation on all of $[0,1]^d$.
Padding to a common depth uses the same noisy identity modules, with their
errors and costs included before the final ReLU output and its bias shift.
\end{proof}

\subsection{Assembly and its exact error and probability budgets}

\begin{proof}[Proof of Theorem~\ref{thm:repaired-upper}]
Use the local Taylor partition from \citet[proof of Theorem 1]{yarotsky2017}.
Write $A=|\{\alpha\in\mathbb N_0^d:|\alpha|<n\}|=
\binom{n+d-1}{d}$ and
\[
 c_{m,\alpha}=D^\alpha f(m/N)/\alpha!,\quad
 P_m(x)=\sum_{|\alpha|<n}c_{m,\alpha}(x-m/N)^\alpha,
 \quad f_N=\sum_m\phi_mP_m.
\]
The derivatives of order below $n$ use their continuous representatives,
including boundary limits; $W^{n,\infty}$ on the cube admits the requisite
$C^{n-1,1}$ representative. In particular $|c_{m,\alpha}|\leq1$.
The partition satisfies $\sum_m\phi_m=1$, has at most $2^d$ nonzero terms
at each point, and its nonzero terms have $\|x-m/N\|_\infty<1/N$.
Taylor's remainder therefore gives
$\|f-f_N\|_\infty\leq C^{\rm Tay}_{n,d}N^{-n}$.
Choose $N=\lceil(2C^{\rm Tay}_{n,d}/\varepsilon)^{1/n}\rceil$, increasing
$C^{\rm Tay}_{n,d}$ to at least one if necessary.

For each $(m,\alpha)$ define
\[
 Q_{m,\alpha}=c_{m,\alpha}\phi_m(x)(x-m/N)^\alpha,
 \qquad q_{m,\alpha,\pm}=(\pm Q_{m,\alpha})_+.
\]
There are $T=2A(N+1)^d$ nonnegative branches. At each $x$, their sum is
at most $A$, because
\[
 \sum_{m,\alpha,\pm}q_{m,\alpha,\pm}(x)
 =\sum_{m,\alpha}|Q_{m,\alpha}(x)|
 \leq A\sum_m\phi_m(x)=A.
\]
There are at most $k=2^dA$ nonzero branches: for each $(m,\alpha)$
at most one of its positive and negative branches is nonzero.

Apply Lemma~\ref{lem:repair-local-term} separately to all branches with
\[
 \eta=\frac{\varepsilon}{8k},\qquad
 \zeta=\frac{\delta}{2T}.
\]
Shift each final output bias down by $\eta$ as in
Lemma~\ref{lem:threshold-localization}. With probability at least
$1-\delta/2$, all resulting outputs $H_i$ satisfy simultaneously
\[
 0\leq H_i\leq q_i,\qquad |H_i-q_i|\leq2\eta,
 \qquad q_i=0\ \Longrightarrow\ H_i=0.
\]
On this event the vector $H(x)$ belongs to
$D_A=\{y\geq0:\sum_i y_i\leq A\}$ for every $x$.
Apply the affine-output form of Lemma~\ref{lem:lb-affine-upper} to the
signed sum $\rho(y)=\sum_{m,\alpha}(y_{m,\alpha,+}-y_{m,\alpha,-})$
on $D_A$, with accuracy $\varepsilon/4$ and failure $\delta/2$.
Its input connections are precisely the first masked edges of this module;
there are no additional unmasked connecting edges.

On the joint event, Lemma~\ref{lem:threshold-composition} gives
\[
 \|F_R-f_N\|_\infty\leq2k\eta+\varepsilon/4=\varepsilon/2.
\]
Adding the Taylor error proves the claimed accuracy and probability.
The total size is at most
\[
 C_{n,d}\frac{T N^2}{p^2\varepsilon^2}
       \log\frac{C_{n,d}T}{\varepsilon\delta}
 +C_{n,d}\frac{T}{p^2\varepsilon^2}\log\frac{C_{n,d}T}{\delta}.
\]
Since $T\leq C_{n,d}N^d$ and $N\leq C_{n,d}\varepsilon^{-1/n}$,
this is \eqref{eq:repaired-upper-size}. All local branches have the same
bounded depth and the outer module adds two layers.
If $f\geq0$, applying ReLU at the existing affine output cannot increase
the error and does not add any edge or layer.
\end{proof}

\section{Cut Sets and Reliability}
\label{sec:reliability-main}
The following bounds complement the accuracy-dependent lower bounds.
They hold for any continuous nonconstant target, without a smoothness
assumption. Their depth parameter is the actual depth of the network,
not an upper bound on the depth that a network is allowed to use.
\begin{theorem}[Cut-set lower bounds]\label{thm:lb-cut}
Let $f$ be continuous and nonconstant, with
$\operatorname{osc}(f)>2\varepsilon$. Let $C$ be a set of $k$ independently
masked edges whose removal disconnects all inputs from the output. If
$\mathbb P(\|F_R-f\|_\infty\leq\varepsilon)\geq1-\delta$, then
\begin{equation}\label{eq:lb-onecut}
 k\geq\frac{\log(1/\delta)}{-\log(1-p)}.
\end{equation}
If all edges join adjacent layers and there are $\ell$ computational layers,
including the output layer, let $E_j$ count edges entering layer $j$ and
$E=\sum_jE_j$. Then
\begin{equation}\label{eq:lb-layercut}
 E\geq\frac{\ell}{-\log(1-p)}
       \log\frac1{1-(1-\delta)^{1/\ell}}.
\end{equation}
For $0<\delta\leq1/2$, the latter is
$\Omega\bigl(\ell\log(\ell/\delta)/[-\log(1-p)]\bigr)$.
Here $\ell$ is the realized depth, not merely an allowed upper bound.
\end{theorem}

\begin{proof}
All edges in $C$ disappear with probability $(1-p)^k$. On that event
the output is constant as a function of the input, even in the presence
of biases. Every constant has uniform error at least
$\operatorname{osc}(f)/2$, so $(1-p)^k\leq\delta$.

For adjacent layers, deleting all edges in any one layer disconnects the
network. These events involve disjoint masks. With $q=1-p$,
\[
 1-\delta\leq\mathbb P(\text{success})
       \leq\prod_{j=1}^{\ell}(1-q^{E_j})
       \leq(1-q^{E/\ell})^{\ell}.
\]
The last inequality follows from concavity of $x\mapsto\log(1-q^x)$
for $x>0$; a zero-sized layer is already incompatible with success.
Solving gives \eqref{eq:lb-layercut}. For the asymptotic statement, put
$a=-\log(1-\delta)$; then $\delta\leq a\leq2\delta$ and
$1-\exp(-a/\ell)$ is comparable to $\delta/\ell$ on this range.
\end{proof}

\section{Algebraic and Numerical Checks}
\label{sec:verification-main}

The proof development was accompanied by six finite verification
scripts. They check the Bernoulli moment identities and the
positive-mask domination argument; midpoint quadrature and the
square and multiplication error constants; bounded subnetworks and
the debiasing identity; constant-node elimination in masked directed
graphs; and the concentration and change-of-measure constants in
the confidence lower bound. These checks are supplementary to the
analytical proofs.

For the multiscale construction, exact rational evaluations verify
the subset debiasing identity for several retention probabilities
and circuits with up to six edges. Piecewise linear evaluations
check that every masked local window and clipped coordinate is
bounded and that the supremum of an affine combination is controlled
by their common breakpoints. The error and probability budgets are
also evaluated in all three regimes $d<2n$, $d=2n$, and $d>2n$.

To check that the Taylor bookkeeping is not restricted to smooth
analytic examples, we additionally use
\[
 f(x)=\frac1{n!}\left(\frac{x_1+\cdots+x_d}{d}-\theta\right)_+^n,
 \qquad \theta\in\{0,1/3,1/2,1\}.
\]
These functions belong to the stated $W^{n,\infty}$ unit ball, while
their highest weak derivatives can be discontinuous. Exact rational
checks with $d=1,2,3$ and $n=1,2,3,4$ verify 22,295 normalized
difference coefficients, 2,304 localized telescoping identities,
and 768 complete multiscale identities, including boundary centers.
The mass and overlap bounds are checked at the same points.

No empirical approximation rate is inferred from these evaluations,
and no full network at the asymptotic size budget is simulated.
The scripts cannot verify an assertion quantified over all targets,
networks, and parameter values. Such assertions require the proofs
given in the preceding sections. The exponent diagram is a plot of
the theoretical functions in the rate comparison, not an experiment.

\bibliographystyle{plainnat}
\bibliography{references}

\begin{thebibliography}{12}
\providecommand{\natexlab}[1]{#1}
\providecommand{\url}[1]{\texttt{#1}}
\expandafter\ifx\csname urlstyle\endcsname\relax
  \providecommand{\doi}[1]{doi: #1}\else
  \providecommand{\doi}{doi: \begingroup \urlstyle{rm}\Url}\fi

\bibitem[Bartlett et~al.(2019)Bartlett, Harvey, Liaw, and
  Mehrabian]{bartlett2019nearly}
Peter~L. Bartlett, Nick Harvey, Christopher Liaw, and Abbas Mehrabian.
\newblock Nearly-tight {VC}-dimension and pseudodimension bounds for piecewise
  linear neural networks.
\newblock \emph{Journal of Machine Learning Research}, 20\penalty0
  (63):\penalty0 1--17, 2019.
\newblock URL \url{https://jmlr.org/papers/v20/17-612.html}.

\bibitem[El~Mhamdi and Guerraoui(2017)]{elmhamdi2017neurons}
El~Mahdi El~Mhamdi and Rachid Guerraoui.
\newblock When neurons fail.
\newblock In \emph{2017 IEEE International Parallel and Distributed Processing
  Symposium}, pages 1028--1037, 2017.
\newblock \doi{10.1109/IPDPS.2017.66}.
\newblock URL \url{https://arxiv.org/abs/1706.08884}.

\bibitem[El-Mhamdi et~al.(2019)El-Mhamdi, Guerraoui, Kucharavy, and
  Volodin]{elmhamdi2019probabilistic}
El-Mahdi El-Mhamdi, Rachid Guerraoui, Andrei Kucharavy, and Sergei Volodin.
\newblock The probabilistic fault tolerance of neural networks in the
  continuous limit, 2019.
\newblock URL \url{https://arxiv.org/abs/1902.01686}.

\bibitem[Manita et~al.(2022)Manita, Peletier, Portegies, Sanders, and
  Senen-Cerda]{manita2022universal}
Oxana~A Manita, Mark~A Peletier, Jacobus~W Portegies, Jaron Sanders, and Albert
  Senen-Cerda.
\newblock Universal approximation in dropout neural networks.
\newblock \emph{Journal of Machine Learning Research}, 23\penalty0
  (19):\penalty0 1--46, 2022.

\bibitem[Ou and B{\"o}lcskei(2026)]{ou2026covering}
Weigutian Ou and Helmut B{\"o}lcskei.
\newblock Covering numbers for deep {ReLU} networks with applications to
  function approximation and nonparametric regression.
\newblock \emph{Foundations of Computational Mathematics}, 2026.
\newblock \doi{10.1007/s10208-026-09762-w}.
\newblock URL
  \url{https://link.springer.com/article/10.1007/s10208-026-09762-w}.
\newblock Published 16 September 2026; arXiv:2410.06378 first submitted 8
  October 2024.

\bibitem[Senen-Cerda and Sanders(2025)]{senencerda2025almost}
Albert Senen-Cerda and Jaron Sanders.
\newblock Almost sure convergence of dropout algorithms for neural networks.
\newblock \emph{Journal of Machine Learning Research}, 26\penalty0
  (283):\penalty0 1--53, 2025.
\newblock URL \url{https://jmlr.org/papers/v26/20-1396.html}.

\bibitem[Siegel(2023)]{siegel2023optimal}
Jonathan~W. Siegel.
\newblock Optimal approximation rates for deep {ReLU} neural networks on
  {Sobolev} and {Besov} spaces.
\newblock \emph{Journal of Machine Learning Research}, 24\penalty0
  (357):\penalty0 1--52, 2023.
\newblock URL \url{https://jmlr.org/papers/v24/23-0025.html}.

\bibitem[Srivastava et~al.(2014)Srivastava, Hinton, Krizhevsky, Sutskever, and
  Salakhutdinov]{srivastava2014dropout}
Nitish Srivastava, Geoffrey Hinton, Alex Krizhevsky, Ilya Sutskever, and Ruslan
  Salakhutdinov.
\newblock Dropout: A simple way to prevent neural networks from overfitting.
\newblock \emph{Journal of Machine Learning Research}, 15\penalty0
  (56):\penalty0 1929--1958, 2014.
\newblock URL \url{https://jmlr.org/papers/v15/srivastava14a.html}.

\bibitem[Wan et~al.(2013)Wan, Zeiler, Zhang, Le~Cun, and
  Fergus]{wan2013dropconnect}
Li~Wan, Matthew Zeiler, Sixin Zhang, Yann Le~Cun, and Rob Fergus.
\newblock Regularization of neural networks using {DropConnect}.
\newblock In \emph{Proceedings of the 30th International Conference on Machine
  Learning}, volume~28 of \emph{Proceedings of Machine Learning Research},
  pages 1058--1066, 2013.
\newblock URL \url{https://proceedings.mlr.press/v28/wan13.html}.

\bibitem[Wu et~al.(2026)Wu, Chen, and Ye]{wu2026deep}
Tong Wu, Nan Chen, and Zhi-Sheng Ye.
\newblock Deep spatially varying coefficient model for interpolation of
  non-stationary meteorological data.
\newblock \emph{Environmetrics}, 37\penalty0 (6):\penalty0 e70129, 2026.
\newblock \doi{10.1002/env.70129}.
\newblock URL \url{https://onlinelibrary.wiley.com/doi/full/10.1002/env.70129}.
\newblock First published online 29 July 2026.

\bibitem[Yarotsky(2017)]{yarotsky2017}
Dmitry Yarotsky.
\newblock Error bounds for approximations with deep relu networks.
\newblock \emph{Neural Networks}, 94:\penalty0 103--114, 2017.

\bibitem[Zhou et~al.(2021)Zhou, Zhuang, Mattina, and
  Whatmough]{zhou2021information}
Chuteng Zhou, Quntao Zhuang, Matthew Mattina, and Paul~N. Whatmough.
\newblock Information contraction in noisy binary neural networks and its
  implications, 2021.
\newblock URL \url{https://arxiv.org/abs/2101.11750}.

\end{thebibliography}
\end{document}